%% file: maxent.tex
\documentclass{article}

     \usepackage[preprint]{neurips_2020}

\usepackage[utf8]{inputenc} 
\usepackage[T1]{fontenc}    
\usepackage{hyperref}       
\usepackage{url}            
\usepackage{booktabs}       
\usepackage{amsfonts}       
\usepackage{nicefrac}       
\usepackage{microtype,setspace}      
\usepackage{times}

\usepackage[colorinlistoftodos, shadow,color=blue!30!white]{todonotes}

\usepackage{amsmath,amsthm,amssymb,amsfonts}

\input{commands}

\title{A maximum-entropy approach to off-policy evaluation in average-reward MDPs}
\author{Nevena Lazi\'c\thanks{DeepMind} \And Dong Yin\footnotemark[1] \And Mehrdad Farajtabar\footnotemark[1] \And Nir Levine\footnotemark[1] \And Dilan G\"or\"ur\footnotemark[1] \And Chris Harris\thanks{Google} \And Dale Schuurmans\footnotemark[2]}

\begin{document}

\maketitle

\begin{abstract}
    This work focuses on off-policy evaluation (OPE) with function approximation in infinite-horizon undiscounted Markov decision processes (MDPs). 
    For MDPs that are ergodic and linear (i.e. where rewards and dynamics are linear in some known features), we provide the first finite-sample OPE error bound, extending
    existing results beyond the episodic and discounted cases.
    In a more general setting, when the feature dynamics are approximately linear and for arbitrary rewards, we propose a new approach for estimating stationary distributions with function approximation. We formulate this problem as finding the maximum-entropy distribution subject to matching feature expectations under empirical dynamics. We show that this results in an exponential-family distribution whose sufficient statistics are the features, paralleling maximum-entropy approaches in supervised learning.  
    We demonstrate the effectiveness of the proposed OPE approaches in multiple environments. 
\end{abstract}

\input{introduction}
\input{problem}

\input{ope}

\input{related}

\input{experiments}

\input{conclusion}

\bibliography{biblio}

\input{appendix}

\end{document}

%% file: commands.tex
\newcommand{\norm}[1]{\left\Vert#1\right\Vert}
\newcommand{\R}{{\mathbb{R}}}

\newcommand{\E}{{\mathbf E}}

\newcommand{\cA}{{\cal A}}

\newcommand{\cD}{{\cal D}}
\newcommand{\cF}{{\cal F}}

\newcommand{\cP}{{\cal P}}

\newcommand{\cS}{{\cal S}}

\newcommand{\vect}{\textsc{vec}} 
 
\allowdisplaybreaks

\newif \ifcomments \commentsfalse

\ifcomments
\newcommand{\mehrdad}[1]{{\color{blue} Mehrdad: #1}}
\newcommand{\nevena}[1]{{\color{purple} Nevena: #1}}
\newcommand{\dong}[1]{{\color{brown} Dong: #1}}
\newcommand{\dilan}[1]{{\color{green} Dilan: #1}}
\newcommand{\chris}[1]{{\color{teal} Chris: #1}}
\newcommand{\nir}[1]{{\color{cyan} Nir: #1}}
\newcommand{\dale}[1]{{\color{orange} Dale: #1}}
\else
\newcommand{\mehrdad}[1]{}
\newcommand{\nevena}[1]{}
\newcommand{\dong}[1]{}
\newcommand{\dilan}[1]{}
\newcommand{\chris}[1]{}
\newcommand{\nir}[1]{}
\newcommand{\dale}[1]{}
\fi
\usepackage[capitalize]{cleveref}

\newcounter{assumption}
\renewcommand{\theassumption}{A\arabic{assumption}}
\newenvironment{assumption}[1][]{\begin{trivlist}\item[] \refstepcounter{assumption}%
 {\bf Assumption\ \theassumption\ {\em (#1)} } }{
 \ifvmode\smallskip\fi\end{trivlist}}
 
 \newtheorem{theorem}{Theorem}[section]
\newtheorem{lemma}[theorem]{Lemma}
\newtheorem{corollary}[theorem]{Corollary}

\newcounter{remark}
\renewcommand{\theremark}{R\arabic{remark}}

\DeclareMathOperator*{\argmin}{arg\,min}

%% file: introduction.tex
\section{Introduction}

Recently, there have been considerable advances in reinforcement learning (RL), with algorithms achieving impressive performance on game playing and simple robotic tasks. Successful approaches typically learn through direct (online) interaction with the environment. However, in many real applications, access to the environment is limited to a fixed dataset, due to considerations of cost, safety, or time. One key challenge in this setting is off-policy evaluation (OPE): the task of evaluating the performance of a target policy given samples collected by a behavior policy.


The focus of our work is OPE in infinite-horizon undiscounted MDPs,  which capture long-horizon tasks such as game playing, routing, and the control of physical systems.
Most recent state-of-the-art OPE  methods for this setting
estimate the ratios of stationary distributions of the target and behavior policy \citep{liu2018breaking, nachum2019dualdice, wen2020batch, nachum2020reinforcement}.
These approaches typically produce estimators that are consistent, but have no finite-sample guarantees, and even the existing guarantees may not hold with function approximation.
One exception is the recent work of \citet{duan2020minimax}, which relies on linear function approximation. They assume that the MDP is linear (i.e. that rewards and dynamics are linear in some known feature space) and analyze OPE in episodic and discounted MDPs when given a fixed dataset of i.i.d. trajectories. They establish a finite-sample instance-dependent error upper bound for regression-based fitted Q-iteration (FQI), and a nearly-matching minimax lower bound. 

Our work extends the results of \citet{duan2020minimax} to the setting of undiscounted ergodic linear MDPs and  non-i.i.d. data (coming from a single trajectory). We provide the first finite-sample OPE error bound for this case; our bound scales similarly to that of \citet{duan2020minimax}, but depends on the MDP mixing time rather than horizon or discount. 
We are not aware of any similar results for off-policy evaluation in average-reward MDPs.
Indeed, while OPE with linear function approximation has been well-studied for discounted MDPs \citep{geist2014off, dann2014policy, yu2010convergence}, in the undiscounted setting even showing convergence of standard methods presents some difficulties (see the discussion in \cite{yu2010tr} for more details). 


Beyond linear MDPs, we consider MDPs in which rewards are non-linear, while the state-action dynamics are still  (approximately) linear in some features.
Here we propose a novel approach for estimating stationary distributions with function approximation: we maximize the distribution entropy subject to matching feature expectations under the empirical dynamics. Interestingly, this results in an exponential family distribution whose sufficient statistics are the features, paralleling  the well-known maximum entropy approach to supervised learning \citep{jaakkola2000maximum}. We demonstrate the effectiveness of our proposed OPE approaches in multiple environments.




%% file: problem.tex
\section{Preliminaries}\label{sec:preliminaries}

{\bf Problem definition.} We are interested in learning from batch data in infinite-horizon ergodic  Markov decision processes (MDPs). An MDP is a tuple $(\cS, \cA, r, P)$, where $\cS$ is the state space, $\cA$ is the action space, $r : \cS \times \cA \rightarrow \R$ is the reward function, and $P : \cS \times \cA \rightarrow \Delta_\cS$ is the transition probability function. 
For ease of exposition, we will assume that states and actions are discrete, but similar ideas apply to continuous state and action spaces.
A policy $\pi : \cS \rightarrow \Delta_\cA$ is a mapping from a state to a distribution over actions.  We will use $\Pi_\pi$ to denote the transition kernel from a state-action pair $(s, a)$ to the next pair $(s', a')$ under $\pi$.
In an ergodic MDP, every policy induces a single recurrent class of states, i.e. any state can be reached from any other state.
The expected average reward of a policy is defined as
\begin{align*}
J_\pi = \lim_{T \rightarrow \infty} \E \left[ \frac{1}{T} \sum_{t=1}^T r(s_t, a_t) \right] \quad {\rm where} \; 
s_{t+1} \sim P(\cdot | s_t, a_t) \;{\rm and} \; a_t \sim \pi(\cdot | s_t)  \,.
\end{align*}
Assume we are given a trajectory of $T$ transitions $\cD_\beta = \{(s_t, a_t, r_t )\}_{t=1}^{T+1}$ generated by a behavior policy $\beta$ in an unknown MDP. 
The \emph{off-policy evaluation} problem is the task of estimating $J_\pi$ for a target policy $\pi$.  


{\bf Stationary distributions.} Let $\mu_\pi(s)$ be the stationary state distribution of a policy $\pi$, and let $d_\pi(s, a) = \mu_\pi(s) \pi(a|s)$ be the stationary state-action distribution. These distributions satisfy the flow constraint 
\begin{align}
\label{eq:flow}
\mu_\pi(s') = \sum_{a'} d_\pi(s', a') = \sum_{s, a} d_\pi(s, a) P(s'|s, a) \,.
\end{align}
The expected average reward can equivalently be written as $J_\pi = \E_{(s, a) \sim d_\pi} [ r(s, a) ]$. Thus, one approach to learning in MDPs from batch data involves estimating  or optimizing $d_\pi$ subject to \eqref{eq:flow}. In particular, given data sampled from $d_\beta$ and distribution estimates $\widehat d_\pi$ and $\widehat d_\beta$, we can estimate $J_\pi$ as 
$
\widehat J_\pi =\frac{1}{T}  \sum_{t=1}^T  \frac{\widehat d_\pi (s_t, a_t) }{\widehat d_\beta (s_t, a_t)} r_t
$, as proposed by \citet{liu2018breaking}.

{\bf Linear MDPs.}  When the state-action space is large or continuous-valued, a common approach to evaluating or optimizing a policy is to use function approximation.  
Define the conditional transition operator $\cP^\pi$ of a policy $\pi$ as 
\begin{align}
\cP^\pi f(s, a) := \E_{s' \sim P(\cdot |s, a), a' \sim \pi(\cdot|s')} [ f(s', a') | s, a] \,.
\label{eq:cond_trans}
\end{align}
With function approximation, it is convenient to assume that for any policy, $\cP^\pi$ operates within a particular function class $\cF$, i.e. for any $f \in \cF$, $\cP^\pi f \in \cF$  \citep{duan2020minimax}.  We will assume that $\cF$ is the set of functions linear in some (known or pre-learned) features $\phi(s, a) \in \R^m$, such that for some matrix $M_\pi \in \R^{m \times m}$, 
\begin{align}
\label{eq:linear}
\cP^\pi \phi(s, a) = \sum_{s', a'} \pi(a' | s') P(s' | s, a) \phi(s', a')^\top =  \phi(s, a)^\top  M_\pi  + b_\pi^\top \,.
\end{align}
Note that, unlike existing work, we specifically include a bias term $b_\pi$ in the above model. When $\phi(s, a)$ is a binary indicator vector for $(s, a)$, $M_\pi$ corresponds to the state-action transition matrix and $b_\pi = 0$. However, $b_\pi$ is non-zero in other settings, such as MDPs with linear-Gaussian dynamics. 
Similarly to \citet{duan2020minimax}, we will assume that rewards $r(s, a)$ are linear in the same features: $r(s, a) = \phi(s, a)^\top w$. This assumption will be required for the purpose of analysis.


%% file: ope.tex
\section{Off-policy evaluation}\label{sec:ope}

\subsection{Maximum-entropy stationary distribution estimation} 
Given a policy $\pi(a|s)$, in order to compute an off-policy estimate of $J_\pi$, we only need to estimate the stationary state distribution $\mu_\pi(s)$.  We formulate this as a maximum-entropy problem subject to matching feature expectations:
\begin{align}
\label{eq:ope_algo}
\min_{\mu \in \Delta_{\cS}} & \; \sum_s \mu(s) \ln \mu(s) 
\\
{\rm s.t.}  & \; 
 \sum_{s', a'} \mu_\pi(s') \pi(a'|s') \phi(s', a') = \sum_{s, a} \mu_\pi(s) \pi(a|s) \sum_{s', a'} P(s'|s, a) \pi(a'|s') \phi(s', a') \,. \label{eq:flow_relaxed}
\end{align}
Note that we have relaxed the original flow constraint \eqref{eq:flow} over all state-action pairs to only require feature expectations to match, similarly to the maximum-entropy principle for supervised learning \citep{jaakkola2000maximum}. 
Furthermore, under the linear MDP assumption and given the model parameters $(M_\pi, b_\pi)$, the feature expectation constraint can be written as 
\begin{align}\label{eq:linear_2}
   \sum_{s} \mu_\pi(s) \phi(s, \pi)^\top (I - M_\pi) =  b_\pi^\top \,,
\end{align}
where $\phi(s, \pi) = \sum_a \pi(a|s) \phi(s, a)$ are feature expectations under the policy.
In Appendix~\ref{sec:ope_lagrangian}, we show that the optimal solution is an exponential-family distribution of the following form:
\begin{align}\label{eq:exp_family_form}
    \mu_\pi(s | \theta_\pi, M_\pi) & = \exp \left(  \phi(s, \pi)^\top (I - M_\pi) \theta_\pi - F(\theta_\pi | M_\pi) \right)
\end{align}
where $F(\theta_\pi | M_\pi)$ is the log-partition function. The parameters $\theta_\pi$ are the solution of the dual problem:
\begin{align}
    \theta_\pi = \argmin_{\theta} D(\theta ) := F(\theta | M_\pi) - \theta^\top b_\pi \,.
\label{eq:theta_optization}
\end{align}
Note that the dual is convex, due to the convexity of the log-partition function in exponential families.
Given a batch of data,
we estimate the stationary distribution $\mu_\pi$ by first estimating $\widehat M_\pi$ and $\hat b_\pi$ using linear regression (see \eqref{eq:m_estimation}), and then computing a parameter estimate  as 
$
\widehat \theta_\pi = \argmin_\theta F(\theta | \widehat M_\pi) - \hat b_\pi^\top \theta
$.
When the log-partition function $F(\theta|\widehat M_\pi)$ is intractable, we can optimize the dual using stochastic gradient descent. Noting that $\nabla_\theta F(\theta | M_\pi) = \E_{\mu_\pi}[ (I - M_\pi^\top) \phi(s, \pi)]$, we can obtain an (almost) unbiased gradient estimate using importance weights: 
\begin{align}
    \widehat{\nabla}_\theta F(\theta | \widehat M_\pi) 
    &\propto \sum_{s \in \cD_\beta} \frac{\hat \mu_\pi(s|\theta, \widehat M_\pi)}{ \hat \mu_\beta(s | \widehat \theta_\beta, \widehat M_\beta)}  (I - \widehat M_\pi^\top)\phi(s, \pi)
    \label{eq:grad_f}
\end{align}
where $\hat \mu_\beta(s| \widehat \theta_\beta, \widehat M_\beta)$ is an estimate of the stationary distribution of the behavior policy computed using the same approach (we assume that the behavior policy is known and otherwise estimate it from the data).
Finally, we evaluate the policy as
\[
\widehat J_\pi = \sum_{t=1}^T \rho_t r_t \, \quad \text {where } \rho_t = \frac{\hat \mu_\pi(s_t) \pi( a_t|s_t)}{\hat \mu_\beta(s_t) \beta(a_t|s_t)}.
\]
In practice, it may be beneficial to normalize the distribution weights $\rho_t$ to sum to 1, known as weighted importance sampling \citep{rubinsteinsimulation,koller2009probabilistic, mahmood2014weighted}. This results in an estimate that is biased but consistent, and often of much lower variance; the same technique can be applied to the gradient weights following \citet{chen2018stochastic}. When 
the log-normalizing constant is intractable, we can normalize the distributions empirically.

\paragraph{Linear rewards.} 
When the rewards are linear in the features, $r(s, a) = \phi(s, a)^\top w$, and $b_\pi \neq {\bf 0}$, there is a faster way to estimate $J_\pi$. Noting that since $J_\pi = \sum_{s, a} d_\pi(s, a)\phi(s, a)^\top w $, we only need to estimate $e_\pi := \sum_{s, a} d_\pi(s, a)\phi(s, a)$ rather than the full distribution. Under the linear MDP assumption,
 $e_\pi^\top = b_\pi^\top (I - M_\pi)^{-1}$. Thus,  given estimates of the model and reward parameters $\widehat M_\pi, \hat b_\pi, \hat w$, we can evaluate the policy as
\begin{align}
\label{eq:closed_form}
    \widehat J_\pi = \hat b_\pi^\top (I - \widehat M_\pi)^{-1} \hat w \,.
\end{align}

\subsection{OPE error analysis.} 
Our analysis requires the following assumptions.
\begin{assumption}[Mixing coefficient]
\label{ass:mixing}
There exists a constant $\kappa > 0$ such that for any state-action distribution  $d$,
\[
\norm{(d_\beta - d)^\top \Pi_\beta}_{1} \leq \exp(-1/\kappa) \norm{d_\beta - d}_1 
\]
where $\Pi_\beta$ is the transition matrix from $(s, a)$ to $(s', a')$ under the policy $\beta$.
\end{assumption}
\begin{assumption}[Bounded linearly independent features]
\label{ass:indep}
Let $\overline \phi(s, a)^\top := [\phi(s, a)^\top \; 1]$. We assume that $\max_{s, a} \norm{\overline \phi(s, a)}_2 \leq C_\Phi$ for some constant $C_\Phi$.  Let $\Phi$ be an $|\cS||\cA| \times (m + 1)$ matrix whose rows are feature vectors $\overline \phi(s, a)$. We assume that the columns of $\Phi$ are linearly independent. 
\end{assumption}
\begin{assumption}[Feature excitation]
\label{ass:features}
For a policy $\pi$ with stationary distribution $d_\pi(s, a)$, define $\Sigma_\pi = \E_{(s, a) \sim d_\pi} [ \overline \phi(s,a)\overline \phi(s, a)^\top ]$. We assume that $\lambda_{\min}(\Sigma_\beta) \geq \sigma > 0$ and  $\lambda_{\min}(\Sigma_\pi) \geq \sigma_\pi > 0$.
\end{assumption}
The above assumptions mean that the exploration policy $\beta(a|s)$ mixes fast and is exploratory, in the sense that the stationary distribution spans all dimensions of the feature space. These assumptions allow us to bound the model error. We also require the evaluated policy to span the feature space for somewhat technical reasons, in order to bound the policy evaluation error. 

Assume that rewards are linear in the features, $r(s, a) =  \phi(s, a)^\top w$. Given a trajectory $\{(s_t, a_t, r_t)\}_{t=1}^{T+1}$, we estimate $M_\pi$, $b_\pi$, and $w$ using regularized least squares:
\begin{align}
    \label{eq:m_estimation}
    \begin{bmatrix} \widehat M_\pi \vspace{3pt} \\ \hat b_\pi^\top \end{bmatrix}
    &= \bigg(\Lambda + \sum_{t=1}^T \overline \phi(s_t, a_t) \overline \phi(s_t, a_t)^\top \bigg)^{-1} \sum_{t=1}^T \overline \phi(s_t, a_t) \phi(s_{t+1}, \pi)^\top \\
     \hat w &= \bigg(\sum_{t=1}^T  \phi(s_t, a_t)  \phi(s_t, a_t)^\top \bigg)^{-1} \sum_{t=1}^T  \phi(s_t, a_t) r_t
\end{align}
where $\Lambda$ is a regularizer and $\overline \phi(s, a) =\left[ \begin{smallmatrix}\phi(s, a)\\ 1 \end{smallmatrix} \right]$.  For the purpose of simplifying the analysis, we let $\Lambda = \alpha \sum_{t=1}^T \overline\phi(s_t, a_t)\overline\phi(s_t, a_t)^\top$; in practice it may be better to use a diagonal matrix.
Let $W_\pi = \left[ \begin{smallmatrix} M_\pi & {\bf 0} \\ b_\pi^\top & 1 \end{smallmatrix} \right]$ and similarly $\widehat W_\pi = \left[ \begin{smallmatrix} \widehat M_\pi & {\bf 0} \\ \hat b_\pi^\top & 1 \end{smallmatrix} \right]$. 
The following Lemma (proven in Appendix~\ref{app:model_error}) bounds the estimation error under Assumptions ~\ref{ass:mixing} and \ref{ass:features} for single-trajectory data:
\begin{lemma}
\label{lemma:model_error}
Let assumptions \ref{ass:mixing}, \ref{ass:indep}, and \ref{ass:features} hold, and let $\alpha = C_\Phi^2 \sigma^{-1}\kappa /  \sqrt{T}$. Then with probability at least $1-\delta$, for constants $C$ and $C_w$,
\begin{align*}
\norm{\widehat W_\pi - W_\pi}_2 &\leq 
C C_\Phi^4 \kappa \sigma^{-2} \sqrt{2 \ln (2(m+1) / \delta) / T} \\
\norm{w - \hat w}_2 &\leq C_w C_\Phi^2 \kappa \sigma^{-2} \sqrt{2 \ln (2m / \delta) / T} \norm{w}_2 \,.
\end{align*}
\end{lemma}

The following theorem bounds the policy evaluation error.
\begin{theorem} [Policy evaluation error]
\label{thm:ope_error}
Let assumptions \ref{ass:mixing}, \ref{ass:indep}, and \ref{ass:features} hold and assume that problem \eqref{eq:ope_algo}-\eqref{eq:flow_relaxed} is feasible. Then, for a constant $C_J$, with probability at least $1-\delta$, the batch policy evaluation error is bounded as
\begin{align}
   |J_\pi - \widehat J_\pi| \leq
 C_J C_\Phi^4 \kappa \sigma_\pi^{-1/2} \sigma^{-2}(1 + \alpha)^2
  \sqrt{2 \ln (2(m+1) / \delta) / T} \norm{w}_2  \,.
\end{align}
\end{theorem}
The proof is given in Appendix~\ref{app:ope_error} and relies on expressing evaluation error in terms of the model error, as well as on the contraction properties of the matrix $(1+\alpha)^{-1}W_\pi$. While we do not provide a lower bound, note that the error scales similarly to the results of \cite{duan2020minimax} for discounted MDPs, which nearly match the corresponding lower bound. 

{\bf Remark 1}. Theorem~\ref{thm:ope_error} holds for any feasible solution $\mu$ of \eqref{eq:ope_algo} and not necessarily just for the maximum-entropy distribution.

{\bf Remark 2}. 
Our results are shown for the case of discrete states and actions and bounded-norm features. In the continuous case, similar conclusions would follow by arguments on the concentration and boundedness of $\Sigma_\beta$ and $\E_{d_\beta, P}[\overline\phi(s, a) \overline \phi(s', \pi)^\top]$.

\subsection{Policy improvement}

The previous sections provides an approach for estimating the average reward, but it is unclear how to perform policy optimization. One possible formulation is to maximize the entropy-regularized expected reward:
\begin{align*}
    \max_{d_\pi \in \Delta^{\cS \times \cA}} &\; \sum_{s, a} d_\pi(s, a) (r(s, a) - \tau \ln d_\pi(s, a))  
    \quad {\rm s.t.} \; \sum_{s, a} d(s, a) \phi(s, a)^\top (I - \widehat M_\pi) = \hat b_\pi \,.
\end{align*}
Unfortunately, this is no longer a convex problem, as the model $(\widehat M_\pi, \hat b_\pi)$ depends on the optimization variables through $\pi(a|s) = \frac{d(s, a)}{\sum_{a'} d(s, a')}$. We describe some ways around this in Appendix~\ref{app:policy_opt}. 

An alternative formulation is to construct a critic for the purpose of policy improvement. 
Let $Q_\pi(s, a)$ be the state-action value function of a policy, corresponding to the value of taking action $a$ in state $s$ and then following the policy forever, and satisfying the Bellman equation
\begin{align}
Q_\pi(s, a) + J_\pi &= r(s, a) + \sum_{s', a'} P(s' |s, a) \pi(a'|s') Q_\pi(s', a')
\end{align} 
The true $Q_\pi$ and $J_\pi$ minimize the Bellman error:
\begin{align}
L_{BE}(Q, J) = \E_{(s, a) \sim d_\pi}\left[(Q(s, a) + J - r(s, a) - \E_{s'\sim P(\cdot|s, a), a'\sim\pi(\cdot| s')}[Q(s', a')])^2 \right] \,.
\end{align}
In the off-policy case, the above expectation can be taken with respect to the stationary distribution of the behavior policy instead and importance-corrected if possible. 
Typically, the Bellman error cannot be minimized directly, as it includes an expectation over unknown $P$, and we only have access to a sample trajectory. However, under the linear MDP assumption, all $Q$-functions are linear, and thus we can estimate these expectations using the feature-dynamics model. Thus we have the following objective for fitting a critic with linear function approximation $Q_\pi(s, a) = \phi(s, a)^\top v_\pi$:
\begin{align}
    L_{MBE}(v, J | \widehat M_\pi, \hat b_\pi) &= \E_{(s, a) \sim d_\pi} [(\phi(s, a)^\top (I - \widehat M_\pi) v - \hat b_\pi^\top v + J - r(s, a))^2] 
\end{align}
Consider instead minimizing the absolute average Bellman error $|L_{AVG}(v, J)|$, where 
\begin{align}
    L_{AVG}(v, J) &:= \E_{(s, a) \sim d_\pi} [ \phi(s, a)^\top (I - \widehat M_\pi) v - b_\pi^\top v + J - r(s, a)] \,.
\end{align}
This error is minimized when $J = \E_{d_\pi}[r(s, a)]$ and $ \E_{d_\pi}[\phi(s, a)^\top (I - \widehat M_\pi) v] = b_\pi^\top v $. The second condition corresponds to driving the gradient of our dual objective to zero. 
The recent work of \citet{xie2020q} suggests that there are some theoretical advantages to minimizing the average rather than squared Bellman error in discounted MDPs. However, in the undiscounted case, it is unclear whether $v$ is useful for policy improvement, as the minimizer is not a function of the reward.
We leave further investigation of policy improvement in average-reward MDPs for future work.

%% file: related.tex
\section{Related work}

The linear MDP assumption along with linear rewards implies that all value functions are linear. 
Thus we first discuss similarities between our approach and common linear action-value function methods in literature, and then give an broader overview of other related work. 

{\bf TD error.} 
The residual gradient algorithm of \citet{baird1995residual} minimizes the mean squared temporal difference (TD) error:
\begin{align}
\label{eq:td}
    L_{TD}(v, J) &= \frac{1}{T} \sum_{t=1}^T \left(\phi(s_t, a_t)^\top v + J - r_t - \phi(s_t', \pi)^\top v \right)^2
\end{align}
It is well-known that this objective is a biased and inconsistent estimate of the true Bellman error \citep{bradtke1996linear}. Correcting the bias requires double samples (two independent samples of $s'$ for the same $(s, a)$ pair), which may not be available in a single-trajectory dataset.
More recent methods rely on fixed point iterations, using the parameters from the previous iteration to construct regression targets. In our setting, fitted Q-iteration (FQI) can be written as
\begin{align}
\label{eq:fqi}
    v^{(k+1)}, J^{(k+1)} &= \argmin_{v, J} \sum_{t=1}^T (\phi(s_t, a_t)^\top v + J - r_t - \phi(s_t', \pi)^\top v^{(k)})^2
\end{align}
The convergence of FQI is guaranteed only in restricted cases (e.g. \cite{antos2008learning}), and no guarantees exist in the undiscounted setting to the best of authors' knowledge.

{\bf PBE error.} Another class of methods 
minimize the projected Bellman error, which corresponds to only the error representable by the features. The advantage of this approach is that the error due to not having exact dynamics expectations is not correlated with the TD error \citep{sutton2009fast}. 
Let $D_\beta = {\rm diag}(d_\beta)$, and let $Q_\pi$ be the action-value function of $\pi$ as a vector. 
In matrix form, the projected Bellman equation (PBE) can be written as
\begin{align}
Q_\pi & = G_\beta ( r - J_\pi {\bf 1} + \Pi_\pi Q_\pi), \quad \text{where } \;G_\beta = \Phi(\Phi^\top D_\beta \Phi)^{-1} \Phi^\top D_\beta  \label{eq:pbe}
\end{align}
where $\Phi$ excludes bias. 
Methods that attempt to solve (the sample-based version of) the above equation using fixed-point iteration may diverge if the matrix $G_\beta \Pi_\pi$ is not contractive (has spectral radius greater than or equal to 1). 
The contractiveness condition has been shown to hold in the on-policy setting ($\beta=\pi$) under mild assumptions by \citet{yu2009convergence}, but need not hold in general. 
Practical solutions may ensure convergence by regularizing $G_\beta$ as $G_\beta^b = \Phi(\Phi^\top D_\beta \Phi + bI)^{-1} \Phi^\top D_\beta$; however, the required bias $b$ may be large.
Another approach, popular in the discounted setting, is to minimize the PBE error using least squares methods like LSTD and LSPE \citep{yu2010convergence, dann2014policy, geist2014off}. 
A general complication in the average-reward case 
is that, unlike with discounted methods which only estimate $Q_\pi$, we also need to solve for $J_\pi$. One possible heuristic  is to initially use a guess for $J_\pi$. However, the resulting projected equation may not have a solution \citep{yu2010convergence, puterman2014markov}, and for OPE, $J_\pi$ is actually the quantity we are after. 

{\bf DualDICE} \citep{nachum2020reinforcement}. DualDICE and related methods  optimize the Lagrangian of the following linear programming (LP) formulation of the Bellman equation:
\begin{align*}
    \min_{J, Q} -J \quad {\rm s.t.} \;\; Q + J {\bf 1} \leq r + \Pi_\pi Q   \,.
\end{align*}
For the average reward case, linear function approximation $Q = \Phi v$, and a linear MDP satisfying $\Pi_\pi \Phi = \Phi W_\pi $, the problem Lagrangian is (see also \cite{zhang2020gendice,nachum2020reinforcement}):
\begin{align}
\label{eq:dice}
    \max_d \min_{J, v} &\; L(d, J, v) = -J + d^\top (\Phi (I- W_\pi) v + J) \,.
\end{align}
We solve an entropy-regularized version of the dual LP using a learned model $\widehat W_\pi$. This can be seen as a particular convex instantiation of DualDICE with function approximation - a linear feature model, linear value functions, and exponential-family stationary distributions. 


{\bf Dual approaches to batch RL.} Many recent methods for batch policy evaluation and optimization rely on estimating stationary distribution ratios  that (approximately) respect the MDP dynamics \citep{liu2018breaking, nachum2019dualdice, nachum2019algaedice, wen2020batch}. In particular, \cite{liu2018breaking} impose a similar constraint to ours on matching feature expectations. However, while we enforce the constraint for a particular feature representation, they minimize the squared error of violating the constraint while maximizing over smooth feature functions in a reproducing kernel Hilbert space.  
\cite{kernel_loss} minimize a kernel loss for solving the Bellman equation. 
We note that our approach can also be kernelized, by using kernel ridge regression in place of linear regression for the model. 
Most of the existing approaches yield consistent estimators, but have no finite-sample guarantees. 
One exception is the work of \citet{duan2020minimax}, which provides a minimax lower bound and nearly-matching finite-sample error bound in linear finite-horizon and discounted MDPs, given a dataset of i.i.d. trajectories. 
Under a similar linearity assumption, we provide a finite-sample OPE error bound for average-cost ergodic MDPs, and our approach only requires a single trajectory of the behavior policy. 


{\bf Maximum-entropy estimation.} The maximum-entropy principle has been well-studied in supervised learning (see e.g. \cite{jaakkola2000maximum}). There the objective is to maximize the entropy of a distribution subject to feature statistics matching on the available data, and the corresponding dual is maximum-likelihood estimation of an exponential family. In the batch RL setting, we maximize entropy subject to feature expectations matching under the MDP dynamics. For the linear MDP, the resulting distribution is also in the exponential family, and parameterized in a particular way that includes the model. 
Existing methods for modeling stationary distributions with function approximation tend to use linear functions and require extra constraints to ensure non-negativity and normalization \citep{cardoso2019large,abbasiyadkori2019largescale}. Exponential families seem like a more elegant solution, and also correspond to well-studied settings such as the linear quadratic regulator. 
\cite{pmlr-v97-hazan19a} proposed learning maximum-entropy stationary distributions for the purpose of exploration. They focused on the tabular MDP case, and required an oracle for solving planning problems with function approximation, since in that case the entropy maximization problem may not be convex. We provide a convex formulation of this problem with function approximation in the linear MDP setting, which can also be used with neural networks (by learning reperesentations).  

%% file: experiments.tex
\section{Experiments}

We compare our approach to other policy evaluation methods relying on function approximation. Since our focus is not on learning representations, we experiment with a fixed linear basis. 
We evaluate fitted Q-iteration (\textsc{FQI}) implemented as in \eqref{eq:fqi} and Bellman residual minimization (\textsc{BRM}) implemented as in \eqref{eq:td}. We also use the average reward of the behavior policy as the simplest baseline. We refer to the closed-form version of our approach in \eqref{eq:closed_form} as \textsc{Model}, and to the version solving for the stationary distribution as \textsc{MaxEnt}.  We regularize the covariances of all  regression problems using $\alpha I$ with tuned $\alpha$.\footnote{Starting with $\alpha = 1$, we keep doubling $\alpha$ for FQI as long as it diverges, and for \textsc{MaxEnt} as long as $|\lambda_{\max}(\widehat W_\pi)| > 1$).} For \textsc{MaxEnt}, we optimize the parameters using full-batch Adam \citep{kingma2014adam},
and normalize the distributions empirically. For experiments with OpenAI Gym  environments \citep{openaigym} (Taxi and Acrobot), we additionally use weighted importance sampling \citep{mahmood2014weighted} for both the gradients and the objective. 
Unless stated otherwise, we generate policies by partially training on-policy using the \textsc{Politex} algorithm \citep{pmlr-v97-lazic19a}, a version of regularized policy iteration with linear Q-functions. We compute the true policy values $J_\pi$ using Monte-Carlo simulation for Acrobot, and exactly for other environments. Overall, we find that using a feature model is helpful with linear value-function methods.
 \begin{figure}[!t]
 \centering
\includegraphics[width=0.325\linewidth, trim=0cm 0.1cm 1cm 0cm, clip]{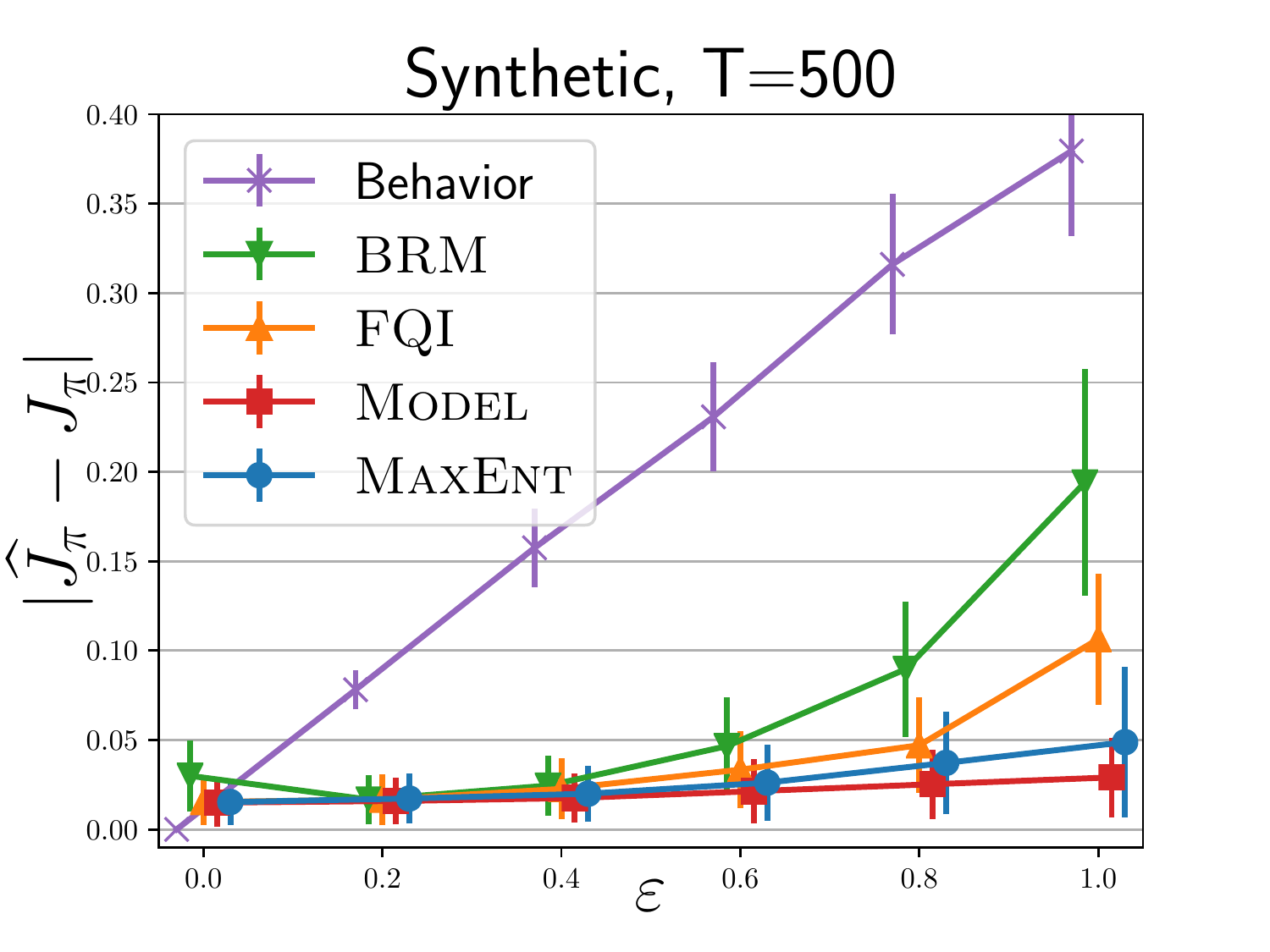}
\hspace{0.1cm}
\includegraphics[width=0.325\linewidth, trim=0cm 0.1cm 1cm 0cm, clip]{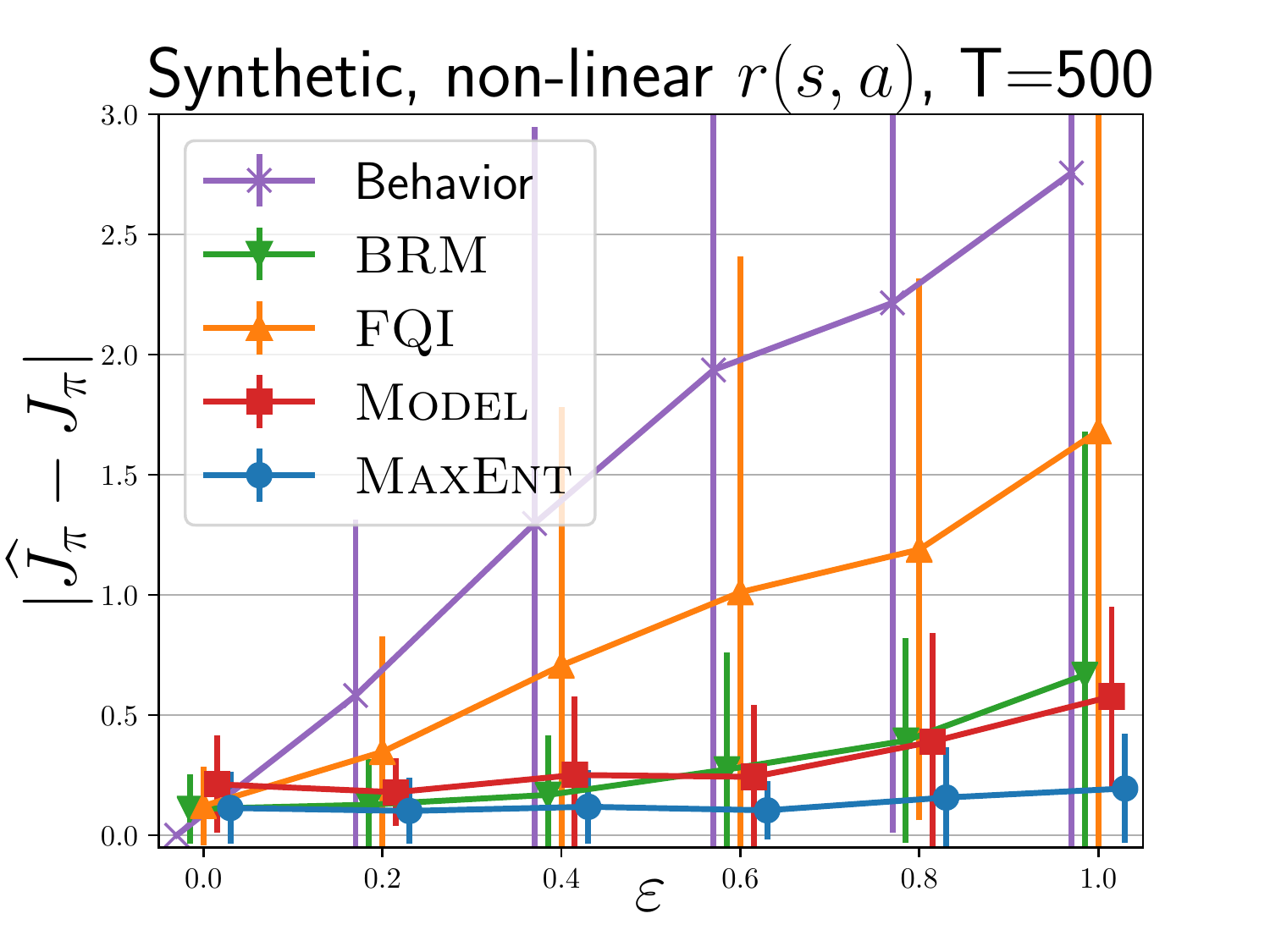}
\includegraphics[width=0.32\linewidth, trim=0cm 0.1cm 1cm 0cm, clip]{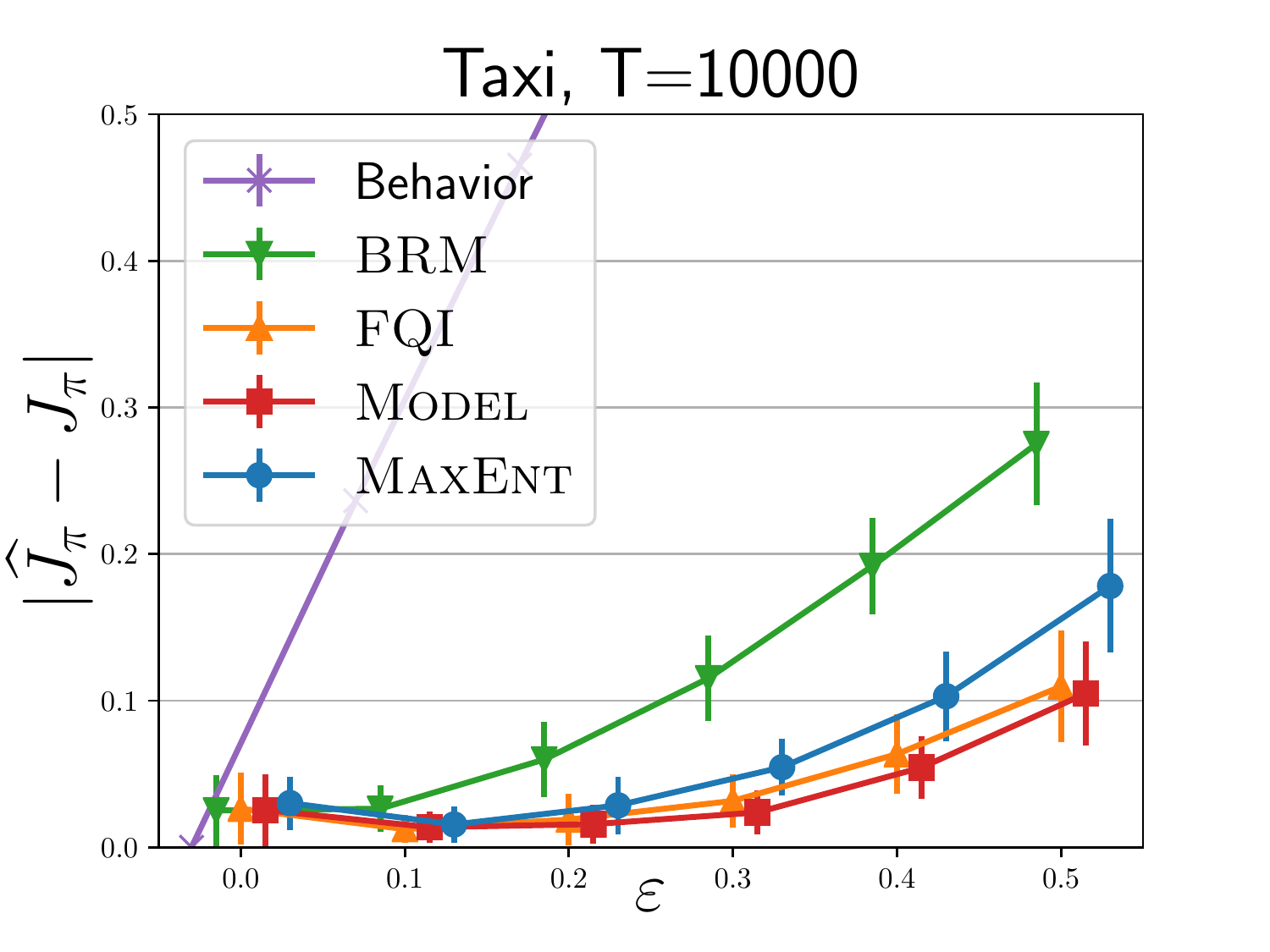}
\caption{Experiments with behavior policy $\varepsilon$-greedy w.r.t. $\pi$ on synthetic environments and Taxi (mean and standard deviation for 100 target policies $\pi$). Note that the plots are slightly shifted along the horizontal axis to make error bars easier to see. 
}
\label{fig:synthetic}
\end{figure}

{\bf Synthetic environments.} We generate synthetic MDPs with 100 states, 10 actions, and transition matrices $P$ generated by sampling entries uniformly at random and normalizing columns to sum to 1. We represent each state with a 10-dimensional vector $\phi_{\cS}(s)$ of random Fourier features \citep{rahimi2008random}, and let $\phi(s, a) = \phi_{\cS}(s) \otimes \phi_{\cA}(a)$, where $\phi_{\cA}(a)$ is a binary indicator vector for action $a$. We experiment with linear rewards $r(s, a) = -\phi(s, a)^\top w$ with entries of $w$ generated uniformly at random, and with non-linear rewards of the form $r(s, a) = -\exp(2\phi(s, a)^\top w)$.  We generate target policies $\pi$ by training on-policy, and set behavior policies $\beta$ to be $\varepsilon$-greedy with respect to $\pi$. We plot the evaluation error $|J_\pi - \widehat J_\pi|$  for several values of $\varepsilon$ in Figure \ref{fig:synthetic} (showing mean and standard deviation for 100 random MDPs for each $\varepsilon$).  We can see that the model-based approaches are less sensitive to the difference between $\pi$ and $\beta$, and the advantage of inferring the full distribution in the non-linear reward case. Note also that the true underlying dynamics are not low-rank, but our low-rank approximation still results in good estimates.

{\bf Taxi.} The Taxi environment \citep{dietterich2000hierarchical} is a $5 \times 5$ grid with four pickup/dropoff locations. Taxi actions include going left, right, up, and down, and picking up or dropping off a passenger. There is a reward of -1 for every step, a reward of -10 for illegal pickup/dropoff actions, and a reward of 20 for a successful dropoff. In the infinite-horizon version, a new passenger appears after a successful dropoff. Our state features include indicators for whether the taxi is empty / at pickup / at dropoff and their pairwise products, and xy-coordinates of the taxi, passenger, and dropoff. 
We set $\pi$ to be $0.05$-greedy w.r.t. a hand-coded optimal strategy, and $\beta$ to be $\varepsilon$-greedy w.r.t. $\pi$. A comparison of different policy evaluation methods is given in Figure~\ref{fig:synthetic}. In this case, all methods are somewhat affected by the suboptimality of the behavior policy, possibly due to fewer successful dropoffs, and \textsc{FQI} and \textsc{Model} perform best.

 \begin{figure}[!t]
 \centering
 \includegraphics[width=0.35\linewidth, trim=0cm 0.1cm 1cm 0cm, clip]{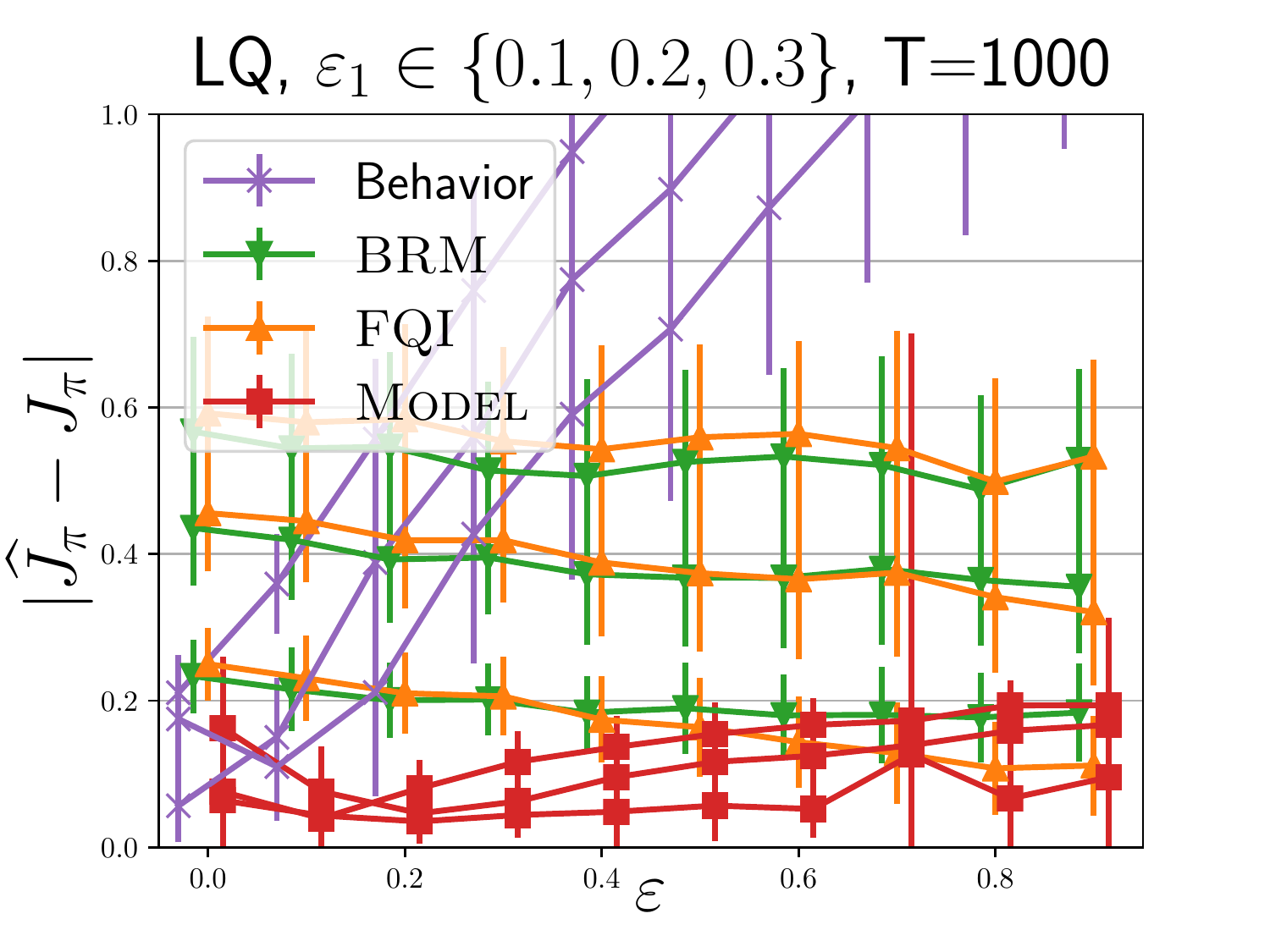}
\hspace{0.1cm}
 \includegraphics[width=0.35\linewidth, trim=0cm 0.1cm 1cm 0cm, clip]{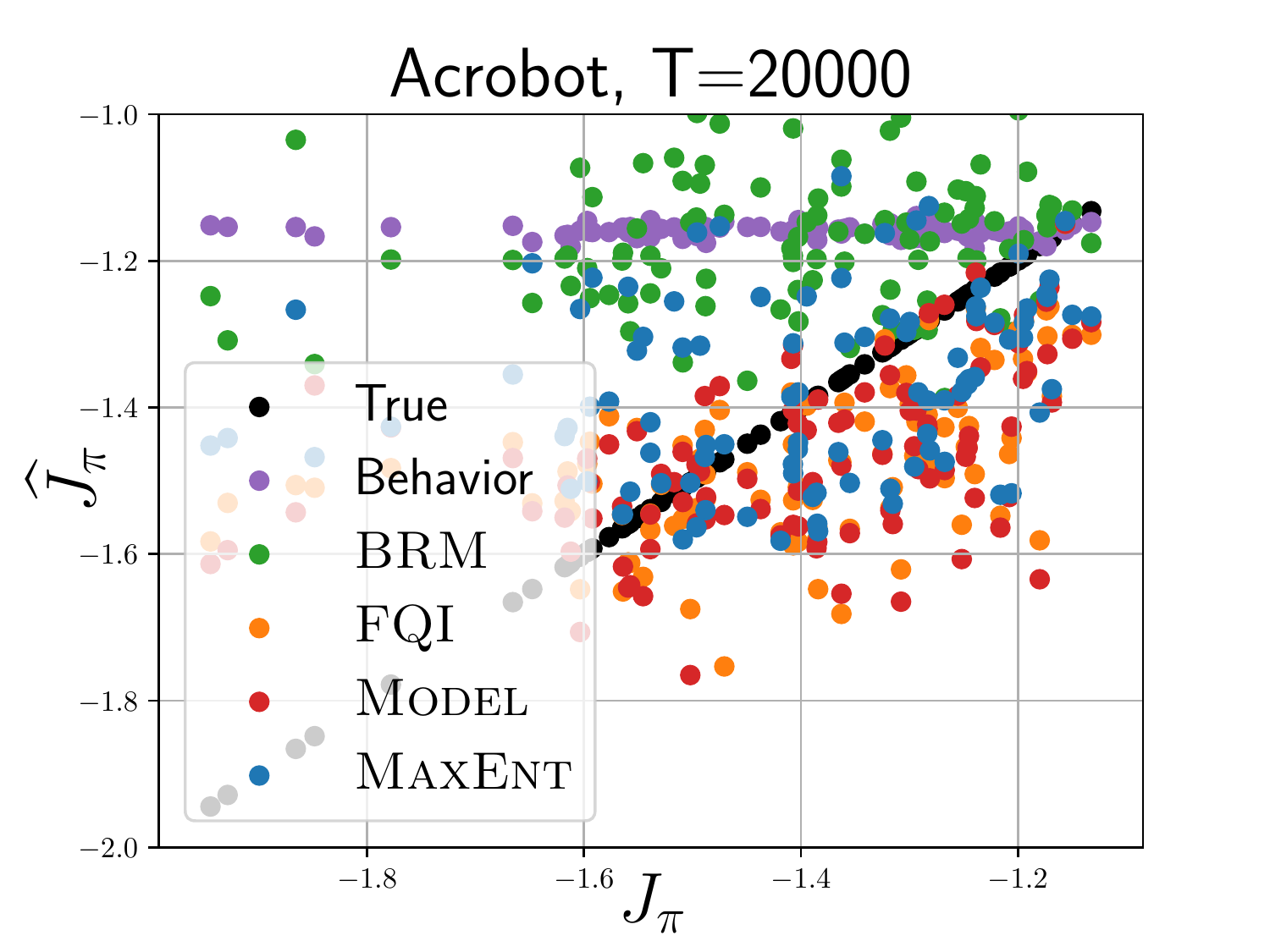}
\caption{Left: experiments on LQ control, where $\varepsilon$ and $\varepsilon_1$ control the suboptimality of the behavior and target policies, respectively. The more suboptimal policies are difficult to evaluate with \textsc{BRM} and \textsc{FQI}. Right: predicted vs. true value on Acrobot for 100 target policies evaluated using the same behavior policy. The errors are: Behavior 0.24 ($\pm 0.17$), \textsc{BRM} 0.23 ($\pm 0.17)$, \textsc{FQI} 0.14 ($\pm 0.10$), \textsc{Model} 0.13 ($\pm 0.11$), \textsc{MaxEnt} 0.15 ($\pm 0.14$). 
}
\label{fig:lqr}
\end{figure}
{\bf Linear quadratic regulator.} 
We evaluate our approach on the linear quadratic (LQ) control system in \cite{dean2019sample}, where stationary distributions, policies, and transition dynamics are Gaussian. We only evaluate a model-based approach here (as well as \textsc{FQI} and \textsc{BRM}) since the model fully constrains the solution, and solve a constrained optimization problem to ensure positive-definite covariances (see Appendix~\ref{app:experiments} for full details). 
We generate policies by solving the optimal control problem for the true dynamics and noisy costs, where $\varepsilon$ controls the noise for $\beta$ and $\varepsilon_1$  controls the noise for $\pi$. The results are shown in Figure~\ref{fig:lqr} (left) for ten values of $\varepsilon$ and three values of $\varepsilon_1$. While $\varepsilon$ does not seem to affect the OPE performance, the error increases with  $\varepsilon_1$ for \textsc{BRM} and \textsc{FQI}. 

{\bf Acrobot} \citep{sutton1996generalization} is a simple episodic discrete-action physical control task. The system includes two links and two joints, one of which is actuated. 
The goal is to swing the lower link up to a given height.
We set the reward at each time step to the negative distance between the joint to its target height, and to 100 when the lower link reaches its target height. Each episode ends after 500 steps, or after the target height is reached, after which we reset.
 The observations are link positions and velocities; we featurize them using the multivariate Fourier basis of order 3 as described in \cite{konidaris2011value}. For this task, we partially train 101 policies, set $\beta$ to the first policy, and evaluate the remaining policies.  
 The results are shown in Figure~\ref{fig:lqr} (right). In this case, \textsc{BRM} predictions seem more correlated to the behavior policy than to the target. The other methods are better correlated with the target policy, but have somewhat high error. 
 Possible reasons for this are the episodic nature of the environment, and the true underlying dynamics being only locally linear.


%% file: conclusion.tex
\section{Conclusion and future work}

We have presented a new approach to batch policy evaluation in average-reward MDPs. For linear MDPs, we have provided a finite-sample bound on the OPE error, which extends the previous results for discounted and episodic settings. In a more general setting with non-linear rewards and approximately linear feature dynamics, we have proposed a maximum-entropy approach to finding stationary distributions with function approximation.
Given that the linear MDP assumption is fairly restrictive, one important direction for future work is extending the framework beyond linear functions. Another direction we are planning to explore is applying this framework to policy optimization. 
Finally, note that the maximum-entropy objective corresponds to minimizing the KL-divergence between the target distribution and the uniform distribution, and we can easily minimize the KL divergence to other distributions instead. While the maximum entropy objective is justified in some cases (see Appendix~\ref{app:entropy}), our formulation allows us to incorporate other prior knowledge and constraints when available, and this is another direction for future work.

\section*{Broader impact}

In general, when learning from a batch of data produced by a fixed behavior policy, we may inherit the biases of that policy, and our models may not generalize beyond the support of the data distribution. In our paper, we circumvent this issue by assuming that the information sufficient for evaluating and optimizing policies is contained in some known features, and that the behavior policy is exploratory enough in the sense that it spans those features. These assumptions may not always hold when applying the method in practice.

%% file: appendix.tex
\newpage
\appendix

\section{Off-policy evaluation dual objective}
\label{sec:ope_lagrangian}

We formulate the estimation of the stationary state distribution $\mu_\pi(s)$ given a policy $\pi(a|s)$ as a maximum-entropy problem subject to matching feature expectations under the linearity assumption:
\begin{align}
\min_{\mu} & \; \sum_s \mu(s) \ln \mu(s) \\
{\rm s.t.} &  \;
\sum_{s, a} \mu(s) \pi(a|s) \phi(s, a)^\top (I - M_\pi) = b_\pi^\top  \label{eq:app_constraint} \\
&  \; \sum_s \mu(s) = 1
\end{align}
Let $\phi(s, \pi) = \sum_a \pi(a|s) \phi(s, a)$ be the expected state features under the target policy. For a fixed / given $M_\pi$, the Lagrangian of the above objective is:
\begin{align*}
L(\mu, \theta, \lambda) = \sum_s \mu(s) \ln \mu(s)  + b_\pi^\top \theta -  \sum_{s} \mu(s)\phi(s, \pi)^\top (I - M_\pi) \theta - \lambda(\sum_s \mu(s) - 1)
\end{align*}
Setting the gradient of $L(\mu, \theta)$ w.r.t. $\mu(s)$ to zero, we get
\begin{align*}
0 &= \ln \mu(s) + 1 - \phi(s, \pi)^\top (I - M_\pi ) \theta - \lambda \\
\mu(s) & = \exp( \phi(s, \pi)^\top (I - M_\pi ) \theta + \lambda - 1)
\end{align*}
Because $\sum_s\mu(s) = 1$, we get
\[
1 - \lambda = \ln \sum_s \exp(\phi(s, \pi)^\top (I - M_\pi) \theta) := F(\theta | M_\pi),
\]
where $F(\theta | M_\pi)$ is the log-normalizer.  
By plugging this expression for $\mu$ into the Lagrangian, we get the following dual maximization objective in $\theta$:
\begin{align*}
D(\theta) &:= \sum_s \mu(s) (\phi(s, \pi)^\top (I - M_\pi ) \theta - F(\theta| M_\pi)) 
 + b_\pi^\top \theta
    - \sum_s \mu(s) \phi(s, \pi)^\top (I - M_\pi ) \theta \\
    &= b_\pi^\top \theta - F(\theta |M_\pi) \,.
\end{align*}

\input{model_error}
\input{ope_error}
\input{stationary_distribution}

\input{experiments_appendix}
\input{policy_opt}

%% file: model_error.tex
\section{Model error}
\label{app:model_error}

\subsection{Preliminaries}

Our error analysis relies on similar techniques as the finite-sample analysis in \citet{pmlr-v97-lazic19a}. We first state some useful results.
\begin{lemma} [Lemma A.1 in \citep{pmlr-v97-lazic19a}]
\label{lem:doob}
Let Assumption~\ref{ass:mixing} hold, 
and
let $\{(s_t,a_t)\}_{t=1}^T$ 
be the state-action sequence obtained when following the behavior policy $\beta$ from an initial distribution $d_0$. For $t\in [T]$,
let $X_t$ be a binary indicator vector with a non-zero at the linear index of the state-action pair $(s_t, a_t)$.
Define for $i \in [T]$,
\begin{align*}
B_i & = \E \left[ \sum_{t=1}^T X_t | X_1, ..., X_i \right], \quad \text{ and }\,\,
B_0  = \E \left[ \sum_{t=1}^T X_t\right].
\end{align*}
Then, $(B_i)_{i=0}^T$ is a vector-valued martingale: $\E[B_i-B_{i-1}|B_0,\dots,B_{i-1}]=0$ for $i=1,\dots,T$,  
and $\norm{B_i - B_{i-1}}_1 \le 4 \kappa$ holds for $i\in [T]$.
\end{lemma}
The constructed martingale is known as the Doob martingale underlying the sum $\sum_{t=1}^T X_t$. %
Let $\Pi_\beta$ be the transition matrix for state-action pairs when following $\beta$. 
Then, for $t=0,\dots,m-1$, $\E[X_{t+1} | X_t] = \Pi_\beta^\top X_t$ and by the Markov property, for any $i\in [T]$,
\begin{align*}
B_i & 
= \sum_{t=1}^i X_t + \sum_{t=i+1}^{T} \E[X_t | X_i] 
= \sum_{t=1}^i X_t + \sum_{t=1}^{T - i} (\Pi_\beta^{t})^{\top} X_i  \quad \text{and} \quad
B_0  = \sum_{t=1}^T (\Pi_\beta^t)^\top X_0 \,.
\end{align*}
It will be useful to define another Doob martingale as follows:
\begin{align}
    Y_i &= \sum_{t=2}^i X_{t-1}X_t^\top + \sum_{t=i+1}^T \E[X_{t-1}X_t^\top |X_i] 
    = \sum_{t=2}^i X_{t-1}X_t^\top + \sum_{t=i+1}^{T-i} {\rm diag}(X_i^\top \Pi_\beta^{t-1}) \Pi_\beta \label{eq:dooby}\\
    Y_0& = \sum_{t=1}^T \E[X_{t-1} X_t^\top ] = \sum_{t=1}^T {\rm diag}(d_0^\top \Pi_\beta^{t-1}) \Pi_\beta \label{eq:dooby0}
\end{align}
where $d_0$ is the initial state-action distribution. 
The difference sequence can again be bounded as $\norm{Y_i - Y_{i-1}}_{1,1} \leq 4 \kappa$ under the mixing assumption (see Appendix D.2.2 of \citet{pmlr-v97-lazic19a} for more details). 

Let $(\cF_k)_k$ be a filtration and define $\E_k[\cdot] := \E[\cdot|\cF_k]$.
\begin{theorem}[Matrix Azuma \citep{tropp2012user}]
\label{thm:azuma}
Consider a finite $(\cF)_k$-adapted sequence $\{X_k\}$ of Hermitian matrices of dimension $m$, and a fixed sequence $\{A_k\}$ of Hermitian matrices that satisfy $\E_{k-1} X_k = 0$ and $X_k^2 \preceq A_k^2$ almost surely.  Let $v = \norm{\sum_k A_k^2}$. Then for all $t \geq 0$, 
\[ P\bigg(\lambda_{\max}\bigg( \sum_k X_k \bigg) \geq t\bigg) \leq m \cdot \exp(-t^2 / 8v) \,. \]
\end{theorem}
Equivalently, with probability at least $1-\delta$, $ \norm{ \sum_k X_k } \leq 2\sqrt{2v \ln (m / \delta) }$.
A version of the inequality for non-Hermitian matrices of dimension $m_1 \times m_2$ can be obtained by applying the theorem to a Hermitian dilation of $X$, 
 $\cD(X) = \big[\begin{smallmatrix}0 & X \\ X^* & 0 \end{smallmatrix} \big]$, which satisfies $\lambda_{\max}(\cD(X)) = \norm{X}$ and $\cD(X)^2 =  \big[\begin{smallmatrix}XX^* & 0 \\ 0 & X^*X \end{smallmatrix} \big]$.  In this case, we have that $v = \max \left(\norm{\sum_k X_k X_k^*}, \norm{\sum_k X_k^* X_k} \right)$. 
 
Let $\Phi$ be a $|\cS||\cA| \times (m+1)$ matrix whose rows correspond to bias-augmented feature vectors $\overline \phi(s, a)$ for each state-action pair $(s, a)$. Let $ \phi_i$ be the feature vector corresponding to the $i^{th}$ row of $\Phi$, and let $C_\Phi = \max_i \norm{ \phi_i}_2$. 
For any matrix $A$, we have 
\begin{align}
\label{eq:bounded_features}
\norm{\Phi^\top A \Phi}_2 = \bigg\| \sum_{ij} A_{ij} \phi_i \phi_j^\top \bigg\|_2 \leq \sum_{i, j} |A_{ij}| \norm{\phi_i \phi_j^\top}_2   \leq C_\Phi^2  \sum_{i, j} |A_{ij}|  = C_\Phi^2 \norm{A}_{1, 1}\,.
\end{align}
 
 \subsection{Proof of Lemma~\ref{lemma:model_error}}
 \begin{proof}
 Let $\Phi$ be a $|\cS||\cA| \times (m+1)$ matrix whose rows correspond to bias-augmented feature vectors $\overline \phi(s, a)$.
 Let $D_\beta = {\rm diag}(d_\beta)$. Let $\tilde d_\beta$ be the empirical data distribution, and $\tilde D_\beta = {\rm diag}(\tilde d_\beta)$. The true and estimated (concatenated) model parameters can be written as
\begin{align*}
    \widehat W_\pi &= (\Lambda + \Phi^\top \tilde D_\beta \Phi)^{-1} \Phi^\top \tilde D_\beta \tilde \Pi_\pi \Phi \\
    W_\pi &= (\Phi^\top D_\beta \Phi)^{-1} \Phi^\top D_\beta \Pi_\pi \Phi
\end{align*}
where $\Pi_\pi$ is the true state-action transition kernel under $\pi$, and $\tilde \Pi_\pi$ corresponds to empirical next-state dynamics $\tilde P$. 
For the true model satisfying $\Pi_\pi \Phi = \Phi M_\pi$, we have taken expectations over $d_\beta$ and taken advantage of Assumption~\ref{ass:features}.

Let $\Lambda = \alpha \Phi^T \tilde D_\beta \Phi$; in this case
\begin{align*}
    \widehat W_\pi &= \frac{1}{1 + \alpha}(\Phi^\top \tilde D_\beta \Phi)^{-1} \Phi^\top \tilde D_\beta \tilde \Pi_\pi \Phi
\end{align*}

We first bound the error for $(1 + \alpha) \widehat M_\pi$. The model error can be upper-bounded as:
\begin{align*}
    \norm{(1 + \alpha) \widehat W_\pi - W_\pi}_2 & \leq
    \norm{(\Phi^\top D_\beta \Phi)^{-1}\Phi^\top (\tilde D_\beta \tilde \Pi_\pi - D_\beta \Pi_\pi) \Phi}_2 \\
    & \;\; + \norm{\big((\Phi^\top \tilde D_\beta \Phi )^{-1} - (\Phi^\top D_\beta \Phi)^{-1} \big) \Phi^\top \tilde D_\beta \Pi_\pi \Phi}_2 \\
    & \leq \sigma^{-1}\norm{\Phi^\top (\tilde D_\beta - D_\beta) \Pi_\pi \Phi}_2 
     + \norm{(\Phi^\top \tilde D_\beta \Phi )^{-1} - (\Phi^\top D_\beta \Phi)^{-1} }_2 \norm{ \Phi^\top \tilde D_\beta \Pi_\pi \Phi}_2 \\
     & \leq \sigma^{-1}\norm{\Phi^\top (\tilde D_\beta - D_\beta) \Pi_\pi \Phi}_2 +
     C_\Phi^2 \norm{(\Phi^\top \tilde D_\beta \Phi )^{-1} - (\Phi^\top D_\beta \Phi)^{-1} }_2
\end{align*}
where the second inequality follows from Assumption~\ref{ass:features}, and the last inequality follows from \eqref{eq:bounded_features}.
We proceed to bound the two terms
\begin{align*}
    E_1 &= \sigma^{-1}\norm{\Phi^\top (\tilde D_\beta \tilde \Pi_\pi - D_\beta \Pi_\pi) \Phi}_2 \\
    E_2 &= \norm{(\Phi^\top \tilde D_\beta \Phi )^{-1} - (\Phi^\top D_\beta \Phi)^{-1} }_2
\end{align*}

{\bf Bounding $E_1$.}
Let $(Y_i)_i$ be the Doob martingale defined in \eqref{eq:dooby}-\eqref{eq:dooby0}, and let $\tilde \Pi_\beta$ be the empirical state-action transition matrix under the policy $\beta$. Note that $\tilde D_\beta \tilde \Pi_\beta = Y_T / T$. Furthermore, let $K^\pi$ be a $|\cS||\cA| \times |\cS||\cA|$ matrix defined as
\[
K^\pi_{(s, a), (s', a')} = 
\begin{cases}
\pi(a'|s) & \text{if } s' = s \\
0 & \text{otherwise}
\end{cases}
\]
Notice that $\tilde D_\beta \tilde \Pi_\beta K^\pi = \tilde D_\beta \tilde \Pi_\pi$ and $D_\beta \Pi_\beta K^\pi = D_\beta \Pi_\pi$. We can upper-bound $E_1$ as:
\begin{align*}
    \sigma^{-1}\norm{\Phi^\top (\tilde D_\beta \tilde \Pi_\beta - D_\beta \Pi_\beta) K^\pi \Phi}_2  
    &= \frac{1}{\sigma T}\norm{\Phi^\top(Y_T - Y_0)K^\pi \Phi}_2 
    + \frac{1}{\sigma T} \norm{ \Phi^\top (Y_0 - T D_\beta \Pi_\beta) K^\pi \Phi}_2
\end{align*}
Note that $\Phi^\top Y_i K^\pi \Phi$ is a matrix-valued martingale, whose difference sequence is bounded by 
\begin{align*}
\norm{(\Phi^\top (Y_i - Y_{i-1}) K^\pi \Phi)^2}_2 & \leq C_\Phi^4 \norm{(Y_i - Y_{i-1}) K^\pi}_{1,1}^2 \leq 16 C_\Phi^4 \kappa ^2
\end{align*}
where we have used \eqref{eq:bounded_features} and the fact that rows of $K^\pi$ sum to 1. 
Applying the matrix-Azuma theorem~\ref{thm:azuma}, we have that with probability at least $\delta$,
\[
\frac{1}{\sigma T}\norm{\Phi^\top(Y_T - Y_0)K^\pi \Phi}_2 \leq 
8 C_\Phi^2 \sigma^{-1}\kappa \sqrt{2 \ln(2(m+1) / \delta) / T} \,.
\]
Using the mixing Assumption~\ref{ass:mixing}, and letting $d_0$ be the initial state-action distribution,
\begin{align*}
    \frac{1}{\sigma T} \norm{ \Phi^\top (Y_0 - T D_\beta \Pi_\beta)K^\pi \Phi}_2 
    & \leq \frac{1}{\sigma T} \sum_{t=1}^T \norm{ \Phi^\top {\rm diag}( d_0^\top \Pi_\beta^t - d_\beta^\top ) \Pi_\beta K^\pi \Phi}_2 \\
    & \leq \frac{C_\Phi^2}{\sigma T} \sum_{t=1}^T \norm{{\rm diag}( d_0^\top \Pi_\beta^t - d_\beta^\top ) \Pi_\pi}_{1,1}   \\
   & \leq \frac{C_\Phi^2}{\sigma T} \sum_{t=1}^T \exp(-t / \kappa) \norm{d_0 - d_\beta}_1 
   \leq \frac{2 C_\Phi^2 \kappa} { \sigma T }
\end{align*}
Thus we get that with probability at least $1 - \delta$,
\begin{align*}
    E_1 & \leq 8 C_\Phi^2 \sigma^{-1}\kappa \left( \sqrt{2 \ln(2(m+1) / \delta) / T} + 1 / T \right)
\end{align*}

{\bf Bounding $E_2$.} To bound $E_2$, we first rely on the Woodbury identity to write
\begin{align*}
 &   (\Phi^\top \tilde D_\beta \Phi )^{-1} - (\Phi^\top D_\beta \Phi)^{-1} \\
    &= (\Phi^\top D_\beta \Phi  + \Phi^\top (D_\beta - \tilde D_\beta) \Phi)^{-1} - (\Phi^\top D_\beta \Phi)^{-1} \\
    &= (\Phi^\top D_\beta \Phi)^{-1}\big( (\Phi^\top D_\beta \Phi)^{-1} + ( \Phi^\top (\tilde D_\beta - D_\beta) \Phi)^{-1}\big)^{-1} (\Phi^\top D_\beta \Phi)^{-1} 
\end{align*}
\begin{align*}
E_2 & \leq \sigma^{-2} \norm{ \big( (\Phi^\top D_\beta \Phi)^{-1} + ( \Phi^\top (\tilde D_\beta - D_\beta) \Phi)^{-1}\big)^{-1}}_2 \\
& \leq \sigma^{-2} \norm{ \Phi^\top (\tilde D_\beta - D_\beta) \Phi}_2 \\
& \leq  8 \sigma^{-2} C_\Phi^2  \kappa \left( \sqrt{2 \ln(2(m+1) / \delta) / T} + 1 / T \right)
\end{align*}
where the second line follows because $ (\Phi^\top D_\beta \Phi)^{-1} \succ 0$, 
and the last line follows by similar concentration arguments as those for $E_1$ for the matrix-valued martingale $\Phi^\top {\rm diag}(B_i) \Phi$, with probability at least $1-\delta$.

{\bf Bounding $\norm{\widehat W_\pi - W_\pi}_2$}. 
Putting previous terms together, with probability at least $1-\delta$, for an absolute constant $C$,
\begin{align}
\label{eq:regularized_error}
    \norm{(1 + \alpha) \widehat W_\pi - W_\pi}_2 \leq C C_\Phi^4 \kappa \sigma^{-2} \sqrt{2 \ln (2(m+1)/\delta) / T}
\end{align}
Furthermore, we have:
\begin{align*}
    \norm{\widehat W_\pi - W_\pi}_2 &\leq \norm{(1 + \alpha) \widehat W_\pi - W_\pi}_2 + \alpha \norm{\widehat W}_2 \\
    & \leq 
    C C_\Phi^4 \kappa \sigma^{-2}\sqrt{2 \ln (2(m+1)/\delta) / T} + \alpha \sigma^{-1} C_\Phi^2
\end{align*}
Setting $\alpha = C_\Phi^2 \sigma^{-1}\kappa / \sqrt{T}$ gives the final result.

{\bf Bounding $\norm{w - \hat w}$.} For linear rewards $r(s, a) = \phi(s, a)^\top w$, we estimate the parameters $w$ using linear regression. 
Abusing notation, assume that the feature matrix $\Phi$ does not include bias for the purpose of this section. 
The true and estimated parameters $w$ and $\hat w$ satisfy
\begin{align}
w &= (\Phi^\top D_\beta \Phi)^{-1} \Phi^\top D_\beta r \\
\hat w &= (\Phi^\top \tilde D_\beta \Phi)^{-1} \Phi^\top \tilde D_\beta r
\end{align}
where $r = \Phi w$ is the length-$|\cS||\cA|$ vector of rewards. We have that
\begin{align}
    \norm{w - \hat w} &= \norm{(\Phi^\top D_\beta \Phi)^{-1} \Phi^\top (D_\beta - \tilde D_\beta) \Phi w} +
    \norm{((\Phi^\top D_\beta \Phi)^{-1}  - (\Phi^\top \tilde D_{\beta} \Phi)^{-1}) \Phi^\top \tilde D_\beta \Phi w}
\end{align}
Using the bounds from the previous section, we get that for a constant $C_w$, with probability at least $1-\delta$,
\begin{align}
    \norm{w - \hat w} \leq C_w C_\Phi^2 \sigma^{-2} \kappa(\sqrt{2 \ln (2m/\delta)/T}) \norm{w}
\end{align}

\end{proof}

%% file: ope_error.tex
\section{Proof of Theorem~\ref{thm:ope_error} (policy evaluation error)}
\label{app:ope_error}

\begin{proof}
Assuming that the optimization problem is feasible, the following holds for the resulting distribution $\hat d_\pi(s, a) = \hat \mu_\pi(s) \pi(a|s)$:
\[
\sum_{s, a} \hat d_\pi(s,a) \overline \phi(s, a)^\top = \sum_{s, a} \hat d_\pi(s,a) \overline \phi(s, a)^\top \widehat W_\pi \,.
\]

Assume that the reward is linear in the features, $r(s, a) = w^\top \phi(s, a)$, and let $\hat w$ be the corresponding parameter estimate. Let $\underline w = \left[\begin{smallmatrix}w \\ 0 \end{smallmatrix}\right]$ and let $\underline{\hat w} = \left[\begin{smallmatrix}\hat w \\ 0 \end{smallmatrix}\right]$.

The policy evaluation error is:
\begin{align*}
J_\pi - \widehat J_\pi &= \sum_{s, a} (d_\pi(s, a) \overline \phi(s, a)^\top \underline w - \hat d_\pi(s, a) \overline \phi(s, a)^\top  \underline {\hat w}) \\
&= \sum_{s, a} (d_\pi(s, a) - \hat d_\pi(s, a)) \overline \phi(s, a)^\top \underline w  + \sum_{s, a} \hat d_\pi(s, a) \phi(s, a)^\top (\underline w -  \underline{\hat w} ) \,.
\end{align*}
The norm of the second term is bounded by $C_\Phi \norm{w - \hat w}_2$. We proceed to bound the first term.

Let $W_\pi^\alpha = \frac{1}{1+\alpha} W_\pi$, and define $e^\top := \sum_{s, a} d(s, a) \overline \phi(s, a)^\top$ and $\hat e^\top := \sum_{s, a} \hat d(s, a) \overline \phi(s, a)^\top$. The first term can be written as:
\begin{align*}
 (e^\top - \hat e^\top) \underline  w  
&= e^\top (W^\alpha_\pi + \alpha W^\alpha_\pi) w - \hat e^\top \widehat W_\pi) \underline w  \\
& = (e - \hat e)^\top W^\alpha_\pi \underline w 
 +  \hat e^\top (W_\pi^\alpha - \widehat W_\pi)  \underline w 
 + \alpha e^\top  W_\pi^\alpha \underline w \\
&=  (e - \hat e)^\top( W^\alpha_\pi)^2 \underline w \\
& \quad +  \hat e^\top (W_\pi^\alpha - \widehat M_\pi)  (I + M^\alpha_\pi) \underline w \\
& \quad + \alpha e^\top  W_\pi ^\alpha ( I + (W_\pi^\alpha)^2)\underline w \\
&= \lim_{K \rightarrow \infty} 
(e - \hat e)^\top( W^\alpha_\pi)^K  \underline w +
\left( 
\hat e^\top (W_\pi^\alpha - \widehat W_\pi)   + 
\alpha e^\top W_\pi^\alpha 
\right)
\left( \sum_{i=0}^{K} (W_\pi^\alpha)^i \right) \underline w  
\end{align*}

In order to evaluate the infinite sum, we first show that $W_\pi$ is non-expansive in a $\Sigma_\pi$-weighted norm (and hence $W^\alpha_\pi$ is contractive):
\begin{align}
    \Sigma_\pi 
    &:= \E_{(s, a) \sim d_\pi}[ \overline \phi(s, a) \overline \phi(s, a)^\top] \notag \\
    &= \E_{(s, a) \sim d_\pi}[ \E_{(s', a') \sim \Pi_\pi(\cdot | s, a)} [ \overline \phi(s', a') \overline \phi(s', a')^\top]] \notag \\
    &= \E_{(s, a) \sim d_\pi}[ W_\pi^\top \overline \phi(s, a) \overline \phi(s, a)^\top W_\pi] + V \notag \\
    &= W_\pi^\top \Sigma_\pi W_\pi + V \label{eq:lyap}
\end{align}
where $V \succeq 0$.
Multiplying each side of \eqref{eq:lyap} by $\Sigma_\pi^{-1/2}$ from the left- and right-hand side, we get that
\begin{align*}
    I &= \Sigma_\pi^{-1/2} W_\pi^\top \Sigma_\pi^{1/2} \Sigma_\pi^{1/2} W_\pi \Sigma_\pi^{-1/2} + \Sigma_\pi^{-1/2}V \Sigma_\pi^{-1/2}\\
    1 & \geq \norm{\Sigma_\pi^{1/2} W_\pi \Sigma_\pi^{-1/2}}_2^2
\end{align*}

Thus we have that $\norm{\Sigma_\pi^{1/2} W^\alpha_\pi \Sigma_\pi^{-1/2}}_2 \leq (1 + \alpha)^{-1}$, and we can compute the infinite sum as: 
\begin{align*}
(W_\pi^\alpha)^i &= \Sigma_\pi^{-1/2} \left( \Sigma_\pi^{1/2} W^\alpha_\pi \Sigma_\pi^{-1/2} \right)^i \Sigma_\pi^{1/2}\\
\norm{\sum_{i=0}^\infty (W_\pi^\alpha)^i } 
& \leq \sum_{i=0}^\infty \norm{\Sigma_\pi^{-1/2}} \norm{ \Sigma_\pi^{1/2} W^\alpha_\pi \Sigma^{-1/2}}^i \norm{\Sigma_\pi^{1/2}}
 \leq  C_\Phi \sigma_{\pi}^{-1/2} (1 + \alpha) 
\end{align*}

The error can now be written as
\begin{align*}
    |J_\pi - \widehat J_\pi | &= 
 \left(\norm{\hat e}_2 \norm{W_\pi^\alpha - \widehat W_\pi}_2 + \alpha \norm{e^\top W_\pi^\alpha}_2 + \norm{u}_2)\right) C_\Phi \sigma_\pi^{-1/2}(1 + \alpha) \norm{w}_2 + C_\Phi \norm{w - \hat w}_2 
\end{align*}
Note that $\norm{e}_2 \leq C_\Phi$ and $\norm{\hat e}_2 \leq C_\phi$. 
Set $\alpha = C_\Phi^2 \sigma^{-1} \kappa /\sqrt{T}$ as in the previous section. 
From \eqref{eq:regularized_error}, we have that
\begin{align}
   \norm{W_\pi^\alpha - \widehat W_\pi}_2 \leq(1 + \alpha) C C_\Phi^4 \kappa \sigma^{-2} \sqrt{2 \ln (2(m+1)/\delta) / T}
\end{align}
Plugging in the model errors and $\alpha$ and combining terms we get the final result in the theorem. 

\end{proof}

%% file: stationary_distribution.tex
\section{Stationary distribution with large entropy}
\label{app:entropy}

In this paper, we try to find the distribution over states that maximizes the entropy under some linear constraints, and use it as a proxy for the stationary distribution. In this section, we provide some theoretical evidence that at least when the probability transition over the states is sufficiently random, the stationary distribution tends to have large entropy.

For simplicity, we focus on finite-state Markov chains instead of MDPs. Consider a Markov chain with state space $\mathcal{S}$ and probability transition matrix $P$. Let $S:=|\mathcal{S}|$. Then the stationary distribution $d$ satisfies $d^\top = d^\top P$. In this section, we assume that each row of $P$ is sampled uniformly at random from the simplex over $\mathcal{S}$, i.e., $\Delta_{\mathcal{S}}$, independently of other rows. We prove the following result, which shows that as $S$ increases, the stationary distribution $d$ converges to a uniform distribution over the states at a rate $\mathcal{O}(1/\sqrt{S})$.

\begin{theorem}\label{thm:stationary_dist}
Let $P$ be the probability transition matrix of a Markov chain with finite state space $\mathcal{S}$, and assume that rows of $P$ are sampled independently and uniformly at random from $\Delta_\mathcal{S}$. Then, with probability at least $1-\delta$, the stationary distribution $d$ of the Markov chain satisfies
\[
\left\|\frac{d}{\|d\|_2} - \frac{1}{\sqrt{S}} {\bf 1}\right\|_2 \le \frac{2\sqrt{10}}{\delta\sqrt{S}},
\]
where ${\bf 1}$ denotes an all-one vector.

\end{theorem}
\begin{proof}
We denote the uniform distribution over the simplex in $\R^S$ by $\mathcal{U}$. The distribution $\mathcal{U}$ is a special case of Dirichlet distribution~\citep{hazewinkel2001dirichlet}. In this proof, we make use of the following properties of $\mathcal{U}$.
\begin{lemma}\label{lem:dirichlet}
~\citep{hazewinkel2001dirichlet} Let $x\sim\mathcal{U}$ and $x_i$ be the $i$-th coordinate of $x$. Then we have
\begin{align*}
    &\E[x_i] = \frac{1}{S},\quad \E[x_i^2] = \frac{2}{S(S+1)},\quad \E[x_i^4] = \frac{24}{S(S+1)(S+2)(S+3)}  \\
    &\E[x_i x_j] = \frac{1}{S(S+1)}, \quad \E[x_i^2x_j^2] = \frac{4}{S(S+1)(S+2)(S+3)},~\forall i\neq j.
\end{align*}
\end{lemma}
This lemma gives us the following direct corollary.
\begin{corollary}\label{cor:dirichlet}
Suppose that $x$ and $y$ are two independent samples from $\mathcal{U}$. Then we have
\begin{align}
    \E[\|x\|_2^2] & = \frac{2}{S+1} \label{eq:exp_x2} \\
    \E[x^\top y] & = \frac{1}{S} \label{eq:exp_xy}  \\
    \E[\|x\|_2^4]  & = \frac{4(S+5)}{(S+1)(S+2)(S+3)}  \label{eq:exp_x4}  \\
    \E[(x^\top y)^2]  &= \frac{S+3}{S(S+1)^2} \label{eq:exp_xy2}
\end{align}
\end{corollary}
\begin{proof}
\[
\E[\|x\|_2^2] = \E\left[\sum_{i=1}^S x_i^2 \right] = \frac{2}{S+1}.
\]
\[
\E[x^\top y] = \E\left[ \sum_{i=1}^S x_iy_i \right] = \sum_{i=1}^S \E[x_i]\E[y_i] = \frac{1}{S}.
\]
\[
\E[\|x\|_2^4] = \E\left[(\sum_{i=1}^S x_i^2)^2\right] = \sum_{i=1}^S\E[x_i^4] + \sum_{i\neq j}\E[x_i^2 x_j^2] = \frac{4(S+5)}{(S+1)(S+2)(S+3)}.
\]
\[
    \E[(x^\top y)^2] = \E\left[(\sum_{i=1}^S x_iy_i)^2\right] = \sum_{i=1}^S \E[x_i^2 y_i^2] + \sum_{i\neq j}\E[x_ix_jy_iy_j] = \frac{S+3}{S(S+1)^2}.
\]
\end{proof}

Now we turn to the proof of Theorem~\ref{thm:stationary_dist}. In the following, we define $\widehat{\Sigma}:=PP^\top$, $\Sigma := \E[\widehat{\Sigma}]$, and let $p_i$ be the $i$-th column of $P^\top$. For a PSD matrix $M$, we define $\lambda_i(M)$ as its $i$-th largest eigenvalue. Since $P$ is a probability transition matrix, we know that $\lambda_1(\widehat{\Sigma}) = 1$, and the corresponding top eigenvector is $\frac{d}{\|d\|_2}$. We then analyze $\Sigma$. Since $\Sigma_{i,j} = \E[p_i^\top p_j]$, according to Corollary~\ref{cor:dirichlet}, we know that $\Sigma_{i,i} = \frac{2}{S+1}$, $\forall i$ and $\Sigma_{i,j} = \frac{1}{S}$, $\forall i\neq j$. Thus
\[
\Sigma = \frac{S-1}{S(S+1)} I + \frac{1}{S}{\bf 1}{\bf 1}^\top.
\]
Then, we know that $\lambda_1(\Sigma) = 1+\frac{S-1}{S(S+1)}$, $\lambda_i(\Sigma) = \frac{S-1}{S(S+1)}$, $\forall i\ge 2$. Then, the gap between the top eigenvalue of $\Sigma$ and the second largest eigenvalue of $\Sigma$ is
\begin{align}\label{eq:eigengap}
\lambda_1(\Sigma) - \lambda_2(\Sigma) = 1.
\end{align}
The top eigenvector of $\Sigma$ is $\frac{1}{\sqrt{S}}{\bf 1}$. Next, we proceed to bound the difference between $\Sigma$ and $\widehat{\Sigma}$. In particular, we bound $\E[\|\widehat{\Sigma} - \Sigma\|_F^2]$. We have
\begin{align}
    \E[\|\widehat{\Sigma} - \Sigma\|_F^2] &= \sum_{i=1}^S \left(\E[\widehat{\Sigma}_{i,i}^2] - \E[\widehat{\Sigma}_{i,i}]^2\right) + \sum_{i\neq j}\left( \E[\widehat{\Sigma}_{i,j}^2] - \E[\widehat{\Sigma}_{i,j}]^2 \right) \nonumber  \\
    &= \sum_{i=1}^S \left(\E[\|p_i\|_2^4] - \E[\|p_i\|_2^2]^2\right) + \sum_{i\neq j}\left( \E[(p_i^\top p_j)^2] - \E[p_i^\top p_j]^2 \right) \nonumber \\
    &=\frac{4S(S-1)}{(S+1)^2(S+2)(S+3)} + \frac{(S-1)^2}{S(S+1)^2} \label{eq:bound_f_norm_1} \\
    &\le \frac{5}{S},  \label{eq:bound_f_norm_2}
\end{align}
where in~\eqref{eq:bound_f_norm_1} we use Corollary~\ref{cor:dirichlet}. Thus, we have
\begin{align}\label{eq:bound_f_norm_3}
    \E[\|\widehat{\Sigma} - \Sigma\|_F] \le \sqrt{\E[\|\widehat{\Sigma} - \Sigma\|_F^2]} \le \sqrt{\frac{5}{S}}.
\end{align}
According to Markov's inequality, with probability at least $1-\delta$,
\begin{align}\label{eq:bound_f_norm_4}
    \|\widehat{\Sigma} - \Sigma\|_F \le \frac{\sqrt{5}}{\delta\sqrt{S}}.
\end{align}
We then apply Davis-Kahan Theorem~\citep{davis1970rotation} (see also Theorem 2 in~\cite{yu2015useful}) and obtain
\[
\sqrt{1 - \langle \frac{d}{\|d\|_2}, \frac{1}{\sqrt{S}}{\bf 1} \rangle^2} \le \frac{2\|\widehat{\Sigma} - \Sigma \|_F}{\lambda_1(\Sigma) - \lambda_2(\Sigma)} = 2\|\widehat{\Sigma} - \Sigma \|_F,
\]
where for the equality we use~\eqref{eq:eigengap}. This implies
\begin{align}
\left \| \frac{d}{\|d\|_2} - \frac{1}{\sqrt{S}}{\bf 1} \right\|_2 &= \sqrt{2 - 2\langle \frac{d}{\|d\|_2}, \frac{1}{\sqrt{S}}{\bf 1} \rangle}  \nonumber \\
&\le \sqrt{2} \sqrt{1 - \langle \frac{d}{\|d\|_2}, \frac{1}{\sqrt{S}}{\bf 1} \rangle^2} \nonumber \\
&\le 2\sqrt{2} \|\widehat{\Sigma} - \Sigma \|_F. \label{eq:pre_final}
\end{align}
Then we can complete the proof by combining~\eqref{eq:bound_f_norm_4} and~\eqref{eq:pre_final}.
\end{proof}

%% file: experiments_appendix.tex
\section{Experiment details for linear quadratic control}

\label{app:experiments}
In a linear-quadratic (LQ) control problem, the dynamics are linear-Gaussian in states $x$:
\begin{align}
    x_{t+1} &= Ax_t + B a_t + w_t, \;\; w_t \sim \mathcal{N}(0, W) \,.
\end{align}
Assume that all policies are linear-Gaussian: $\pi(a|x) = \mathcal{N}(a|Kx, C)$. In this case, assuming that the policy $\pi$ is stable (the spectral radius of $A+BK$ is less than 1), the stationary state distribution is
\begin{align}
    \mu(x) = \mathcal{N}(0, S), \qquad {\rm where} \;\; S= (A+BK) S (A+BK)^\top + W\,.
\end{align}
Given an estimate of the  dynamics parameters $(\widehat A, \widehat B, \widehat W)$, maximum-entropy OPE corresponds to the following convex problem:
\begin{align}
    \max_{S \succeq 0} & \;\;\ln \det (S) \\
    {\rm s.t.} & \;\;S = (\widehat A + \widehat B K) S (\widehat A + \widehat B K)^\top + \widehat W 
\end{align}
Note that the constraint corresponds to that in the dual formulation of LQ control presented in \cite{pmlr-v80-cohen18b}.
We solve the above problem using cvxpy \citep{diamond2016cvxpy}. The problem will only be feasible if $\rho(\widehat A + \widehat BK) < 1$, where $\rho(\cdot)$ denotes the spectral radius. Furthermore, when the system is controllable, the constraint fully specifies the solution and so the maximum-entropy objective plays no role. 

In LQ control problems, rewards are quadratic: 
\begin{align}
    r(x, a) = -x^\top Q x -a^\top R a, \quad Q, R \succ 0,
\end{align}
Thus to evaluate policies, we can estimate $Q$ and $R$, and estimate the policy value as 
\[
\widehat J_\pi = {\rm trace}(S \hat Q) + {\rm trace}(( K S K^\top + C) \hat R) \,.
\]

In our experimental setup, we produce the behavior policies by solving for the optimal controller for true dynamics $(A, B, W)$, true action costs $R$, and state costs corrupted as
\begin{align*}
    \tilde Q = Q + \varepsilon^2 U^\top U
\end{align*}
 $U$ is a matrix of the same size as $Q$ whose entries are generated uniformly at random. Given the corresponding optimal linear feedback matrices $\tilde K$, we set behavior policies to $\beta(a|x) = \mathcal{N}(a| \tilde Kx, 0.1 I)$, and we make target policies greedy, i.e. $a = \tilde Kx$. 

When evaluating policies using \textsc{BRM} and \textsc{FQI}, we use the following features for a policy 
$\pi(a|x) = \mathcal{N}(a|Kx, C)$:
\begin{align*}
    \phi(s, a) &= \vect\left(\begin{bmatrix} xx^\top  & xa^\top \\ ax^\top & aa^\top\end{bmatrix} 
   \right), \qquad 
    \phi(s, \pi) = \vect\left( 
    \begin{bmatrix} xx^\top & xx^\top K^\top \\ Kxx^\top  & Kxx^\top K^\top + C 
    \end{bmatrix}     
    \right) \,.
\end{align*}

%% file: policy_opt.tex
\section{Batch policy optimization}
\label{app:policy_opt}

To optimize policies, we can maximize the entropy-regularized expected reward $\sum_{s,a}d(s, a)(r(s, a) - \tau \ln d(s, a))$ subject to the same feature constraints as before. Unfortunately, in this case the feature expectation constraints are no longer linear, as the model $\widehat M_\pi$ depends on the optimization variables through $\pi(a|s) = \frac{d(s, a)}{\sum_{a'} d(s, a')}$.
One possible optimization approach is an EM-like algorithm that alternates between optimizing $d(s, a)$ for fixed $\widehat M_\pi$, and reestimating $\widehat M_\pi$ for $\pi(a|s) \propto d(s, a)$.
A simpler alternative, proposed in \cite{yang2019reinforcement}, is to assume that we have state-only features $\psi(s)$ whose expectation is a linear function of the state-action features:
\begin{align*}
    \E_{s' \sim P(\cdot | s, a)} [ \psi(s')] = \phi(s, a)^\top M
\end{align*}
where $M$ is a matrix of appropriate dimensions. Note that now $M$ does not depend on the policy, and can be kernelized as in  \citet{yang2019reinforcement}. With this, we formulate batch policy optimization as:
\begin{align}
    \min_{d \in \Delta_{\cS \times \cA}} & \; \sum_{s, a} d(s, a) (- 
    \phi(s, a)^\top \hat w + \tau \ln d(s, a)) \\
{\rm s.t.} & \; 
\sum_{s, a} d(s, a) \phi(s, a)^\top \widehat M = \sum_{s, a} d(s, a) \psi(s)^\top
\end{align}
The optimal solution takes the form
\begin{align}
    d(s, a| \theta_d, \hat w, \widehat M) &= \exp\left(\frac{1}{\tau} \phi(s, a)^\top \hat w + \frac{1}{\tau}(\phi(s, a)^\top \widehat M - \psi(s)^\top) \theta_d - F_\tau(\theta_d, \hat w, \widehat M)\right)
\end{align}
where $F_\tau(\theta_d, \hat w, \widehat M)$ is the log-partition function, and $\theta_d = \arg \min_\theta F(\theta, \hat w, \widehat M)$.

%% file: maxent.bbl
\begin{thebibliography}{42}
\providecommand{\natexlab}[1]{#1}
\providecommand{\url}[1]{\texttt{#1}}
\expandafter\ifx\csname urlstyle\endcsname\relax
  \providecommand{\doi}[1]{doi: #1}\else
  \providecommand{\doi}{doi: \begingroup \urlstyle{rm}\Url}\fi

\bibitem[Abbasi-Yadkori et~al.(2019{\natexlab{a}})Abbasi-Yadkori, Bartlett,
  Bhatia, Lazi\'c, Szepesv\'ari, and Weisz]{pmlr-v97-lazic19a}
Yasin Abbasi-Yadkori, Peter Bartlett, Kush Bhatia, Nevena Lazi\'c, Csaba
  Szepesv\'ari, and Gell\'ert Weisz.
\newblock {POLITEX}: Regret bounds for policy iteration using expert
  prediction.
\newblock In \emph{Proceedings of the 36th International Conference on Machine
  Learning}, 2019{\natexlab{a}}.

\bibitem[Abbasi-Yadkori et~al.(2019{\natexlab{b}})Abbasi-Yadkori, Bartlett,
  Chen, and Malek]{abbasiyadkori2019largescale}
Yasin Abbasi-Yadkori, Peter~L. Bartlett, Xi~Chen, and Alan Malek.
\newblock Large-scale markov decision problems via the linear programming dual.
\newblock \emph{arXiv preprint arXiv:1901.01992}, 2019{\natexlab{b}}.

\bibitem[Antos et~al.(2008)Antos, Szepesv{\'a}ri, and Munos]{antos2008learning}
Andr{\'a}s Antos, Csaba Szepesv{\'a}ri, and R{\'e}mi Munos.
\newblock Learning near-optimal policies with bellman-residual minimization
  based fitted policy iteration and a single sample path.
\newblock \emph{Machine Learning}, 71\penalty0 (1):\penalty0 89--129, 2008.

\bibitem[Baird(1995)]{baird1995residual}
Leemon Baird.
\newblock Residual algorithms: Reinforcement learning with function
  approximation.
\newblock In \emph{Machine Learning Proceedings 1995}, pages 30--37. Elsevier,
  1995.

\bibitem[Bradtke and Barto(1996)]{bradtke1996linear}
Steven~J Bradtke and Andrew~G Barto.
\newblock Linear least-squares algorithms for temporal difference learning.
\newblock \emph{Machine learning}, 22\penalty0 (1-3):\penalty0 33--57, 1996.

\bibitem[Brockman et~al.(2016)Brockman, Cheung, Pettersson, Schneider,
  Schulman, Tang, and Zaremba]{openaigym}
Greg Brockman, Vicki Cheung, Ludwig Pettersson, Jonas Schneider, John Schulman,
  Jie Tang, and Wojciech Zaremba.
\newblock {OpenAI Gym}, 2016.

\bibitem[Chen and Luss(2018)]{chen2018stochastic}
Jie Chen and Ronny Luss.
\newblock Stochastic gradient descent with biased but consistent gradient
  estimators.
\newblock \emph{arXiv preprint arXiv:1807.11880}, 2018.

\bibitem[Cohen et~al.(2018)Cohen, Hasidim, Koren, Lazi\'c, Mansour, and
  Talwar]{pmlr-v80-cohen18b}
Alon Cohen, Avinatan Hasidim, Tomer Koren, Nevena Lazi\'c, Yishay Mansour, and
  Kunal Talwar.
\newblock Online linear quadratic control.
\newblock In \emph{Proceedings of the 35th International Conference on Machine
  Learning}, volume~80 of \emph{Proceedings of Machine Learning Research},
  pages 1029--1038. PMLR, 10--15 Jul 2018.

\bibitem[Dann et~al.(2014)Dann, Neumann, Peters, et~al.]{dann2014policy}
Christoph Dann, Gerhard Neumann, Jan Peters, et~al.
\newblock Policy evaluation with temporal differences: A survey and comparison.
\newblock \emph{Journal of Machine Learning Research}, 15:\penalty0 809--883,
  2014.

\bibitem[Davis and Kahan(1970)]{davis1970rotation}
Chandler Davis and William~Morton Kahan.
\newblock The rotation of eigenvectors by a perturbation. iii.
\newblock \emph{SIAM Journal on Numerical Analysis}, 7\penalty0 (1):\penalty0
  1--46, 1970.

\bibitem[Dean et~al.(2019)Dean, Mania, Matni, Recht, and Tu]{dean2019sample}
Sarah Dean, Horia Mania, Nikolai Matni, Benjamin Recht, and Stephen Tu.
\newblock On the sample complexity of the linear quadratic regulator.
\newblock \emph{Foundations of Computational Mathematics}, 2019.

\bibitem[Diamond and Boyd(2016)]{diamond2016cvxpy}
Steven Diamond and Stephen Boyd.
\newblock {CVXPY}: {A} {P}ython-embedded modeling language for convex
  optimization.
\newblock \emph{Journal of Machine Learning Research}, 17\penalty0
  (83):\penalty0 1--5, 2016.

\bibitem[Dietterich(2000)]{dietterich2000hierarchical}
Thomas~G Dietterich.
\newblock {Hierarchical reinforcement learning with the MAXQ value function
  decomposition}.
\newblock \emph{{Journal of Artificial Intelligence Research}}, 13:\penalty0
  227--303, 2000.

\bibitem[Duan and Wang(2020)]{duan2020minimax}
Yaqi Duan and Mengdi Wang.
\newblock Minimax-optimal off-policy evaluation with linear function
  approximation.
\newblock \emph{arXiv preprint arXiv:2002.09516}, 2020.

\bibitem[Feng et~al.(2019)Feng, Li, and Liu]{kernel_loss}
Yihao Feng, Lihong Li, and Qiang Liu.
\newblock {A Kernel Loss for Solving the Bellman Equation}.
\newblock In \emph{{Advances in Neural Information Processing Systems 32}},
  pages 15456--15467. Curran Associates, Inc., 2019.

\bibitem[Geist and Scherrer(2014)]{geist2014off}
Matthieu Geist and Bruno Scherrer.
\newblock Off-policy learning with eligibility traces: A survey.
\newblock \emph{Journal of Machine Learning Research}, 15\penalty0
  (1):\penalty0 289--333, 2014.

\bibitem[Hazan et~al.(2019)Hazan, Kakade, Singh, and
  Van~Soest]{pmlr-v97-hazan19a}
Elad Hazan, Sham Kakade, Karan Singh, and Abby Van~Soest.
\newblock Provably efficient maximum entropy exploration.
\newblock In \emph{Proceedings of the 36th International Conference on Machine
  Learning}, volume~97 of \emph{Proceedings of Machine Learning Research},
  pages 2681--2691. PMLR, 2019.

\bibitem[Hazewinkel(2001)]{hazewinkel2001dirichlet}
M~Hazewinkel.
\newblock Dirichlet distribution.
\newblock \emph{Encyclopedia of Mathematics}, 2001.

\bibitem[Jaakkola et~al.(2000)Jaakkola, Meila, and Jebara]{jaakkola2000maximum}
Tommi Jaakkola, Marina Meila, and Tony Jebara.
\newblock Maximum entropy discrimination.
\newblock In \emph{{Advances in Neural Information Processing Systems}}, pages
  470--476, 2000.

\bibitem[Kingma and Ba(2014)]{kingma2014adam}
Diederik~P Kingma and Jimmy Ba.
\newblock Adam: A method for stochastic optimization.
\newblock \emph{arXiv preprint arXiv:1412.6980}, 2014.

\bibitem[Koller and Friedman(2009)]{koller2009probabilistic}
Daphne Koller and Nir Friedman.
\newblock \emph{Probabilistic graphical models: principles and techniques}.
\newblock MIT press, 2009.

\bibitem[Konidaris et~al.(2011)Konidaris, Osentoski, and
  Thomas]{konidaris2011value}
George Konidaris, Sarah Osentoski, and Philip Thomas.
\newblock {Value function approximation in reinforcement learning using the
  Fourier basis}.
\newblock In \emph{{25th AAAI Conference on Artificial Intelligence}}, 2011.

\bibitem[Liu et~al.(2018)Liu, Li, Tang, and Zhou]{liu2018breaking}
Qiang Liu, Lihong Li, Ziyang Tang, and Dengyong Zhou.
\newblock Breaking the curse of horizon: Infinite-horizon off-policy
  estimation.
\newblock In \emph{{Advances in Neural Information Processing Systems}}, pages
  5356--5366, 2018.

\bibitem[Mahmood et~al.(2014)Mahmood, van Hasselt, and
  Sutton]{mahmood2014weighted}
A~Rupam Mahmood, Hado~P van Hasselt, and Richard~S Sutton.
\newblock Weighted importance sampling for off-policy learning with linear
  function approximation.
\newblock In \emph{{Advances in Neural Information Processing Systems}}, pages
  3014--3022, 2014.

\bibitem[Nachum and Dai(2020)]{nachum2020reinforcement}
Ofir Nachum and Bo~Dai.
\newblock Reinforcement learning via fenchel-rockafellar duality.
\newblock \emph{arXiv preprint arXiv:2001.01866}, 2020.

\bibitem[Nachum et~al.(2019{\natexlab{a}})Nachum, Chow, Dai, and
  Li]{nachum2019dualdice}
Ofir Nachum, Yinlam Chow, Bo~Dai, and Lihong Li.
\newblock {Dual{DICE}: Behavior-agnostic estimation of discounted stationary
  distribution corrections}.
\newblock In \emph{{Advances in Neural Information Processing Systems}}, pages
  2315--2325, 2019{\natexlab{a}}.

\bibitem[Nachum et~al.(2019{\natexlab{b}})Nachum, Dai, Kostrikov, Chow, Li, and
  Schuurmans]{nachum2019algaedice}
Ofir Nachum, Bo~Dai, Ilya Kostrikov, Yinlam Chow, Lihong Li, and Dale
  Schuurmans.
\newblock {AlgaeDICE: Policy Gradient from Arbitrary Experience}.
\newblock \emph{arXiv preprint arXiv:1912.02074}, 2019{\natexlab{b}}.

\bibitem[Puterman(2014)]{puterman2014markov}
Martin~L Puterman.
\newblock \emph{Markov decision processes: discrete stochastic dynamic
  programming}.
\newblock John Wiley \& Sons, 2014.

\bibitem[Rahimi and Recht(2008)]{rahimi2008random}
Ali Rahimi and Benjamin Recht.
\newblock Random features for large-scale kernel machines.
\newblock In \emph{{Advances in Neural Information Processing Systems}}, pages
  1177--1184, 2008.

\bibitem[Rivera~Cardoso et~al.(2019)Rivera~Cardoso, Wang, and
  Xu]{cardoso2019large}
Adrian Rivera~Cardoso, He~Wang, and Huan Xu.
\newblock {Large-scale Markov Decision Processes with changing rewards}.
\newblock In \emph{{Advances in Neural Information Processing Systems 32}},
  pages 2340--2350. 2019.

\bibitem[Rubinstein(1981)]{rubinsteinsimulation}
RY~Rubinstein.
\newblock {Simulation and the Monte Carlo method}.
\newblock 1981.

\bibitem[Sutton(1996)]{sutton1996generalization}
Richard~S Sutton.
\newblock Generalization in reinforcement learning: Successful examples using
  sparse coarse coding.
\newblock In \emph{{Advances in Neural Information Processing Systems}}, pages
  1038--1044, 1996.

\bibitem[Sutton et~al.(2009)Sutton, Maei, Precup, Bhatnagar, Silver,
  Szepesv{\'a}ri, and Wiewiora]{sutton2009fast}
Richard~S Sutton, Hamid~Reza Maei, Doina Precup, Shalabh Bhatnagar, David
  Silver, Csaba Szepesv{\'a}ri, and Eric Wiewiora.
\newblock Fast gradient-descent methods for temporal-difference learning with
  linear function approximation.
\newblock In \emph{Proceedings of the 26th International Conference on Machine
  Learning}, pages 993--1000, 2009.

\bibitem[Tropp(2012)]{tropp2012user}
Joel~A Tropp.
\newblock User-friendly tail bounds for sums of random matrices.
\newblock \emph{Foundations of computational mathematics}, 12\penalty0
  (4):\penalty0 389--434, 2012.

\bibitem[Wen et~al.(2020)Wen, Dai, Li, and Schuurmans]{wen2020batch}
Junfeng Wen, Bo~Dai, Lihong Li, and Dale Schuurmans.
\newblock Batch stationary distribution estimation.
\newblock \emph{arXiv preprint arXiv:2003.00722}, 2020.

\bibitem[Xie and Jiang(2020)]{xie2020q}
Tengyang Xie and Nan Jiang.
\newblock {Q* Approximation Schemes for Batch Reinforcement Learning: A
  Theoretical Comparison}, 2020.

\bibitem[Yang and Wang(2019)]{yang2019reinforcement}
Lin~F Yang and Mengdi Wang.
\newblock Reinforcement learning in feature space: Matrix bandit, kernels, and
  regret bound.
\newblock \emph{arXiv preprint arXiv:1905.10389}, 2019.

\bibitem[Yu(2010{\natexlab{a}})]{yu2010convergence}
Huizhen Yu.
\newblock {Convergence of Least Squares Temporal Difference Methods Under
  General Conditions.}
\newblock In \emph{Proceedings of the International Conference on Machine
  Learning}, pages 1207--1214, 2010{\natexlab{a}}.

\bibitem[Yu(2010{\natexlab{b}})]{yu2010tr}
Huizhen Yu.
\newblock {Convergence of Least Squares Temporal Difference Methods Under
  General Conditions.}
\newblock Technical report, University of Helsinki, Department of Computer
  Science, 04 2010{\natexlab{b}}.

\bibitem[Yu and Bertsekas(2009)]{yu2009convergence}
Huizhen Yu and Dimitri~P Bertsekas.
\newblock Convergence results for some temporal difference methods based on
  least squares.
\newblock \emph{IEEE Transactions on Automatic Control}, 54\penalty0
  (7):\penalty0 1515--1531, 2009.

\bibitem[Yu et~al.(2015)Yu, Wang, and Samworth]{yu2015useful}
Yi~Yu, Tengyao Wang, and Richard~J Samworth.
\newblock {A useful variant of the Davis--Kahan theorem for statisticians}.
\newblock \emph{Biometrika}, 102\penalty0 (2):\penalty0 315--323, 2015.

\bibitem[Zhang et~al.(2020)Zhang, Dai, Li, and Schuurmans]{zhang2020gendice}
Ruiyi Zhang, Bo~Dai, Lihong Li, and Dale Schuurmans.
\newblock Gen{DICE}: Generalized offline estimation of stationary values.
\newblock \emph{arXiv preprint arXiv:2002.09072}, 2020.

\end{thebibliography}
